\documentclass[3p]{elsarticle}

\usepackage{lineno,hyperref}
\modulolinenumbers[5]
\usepackage{amsmath,amssymb,amsthm}

\theoremstyle{definition} 
\newtheorem{definition}{Definition}

\newtheorem{proposition}{Proposition}
\usepackage{multirow} 
\usepackage{makecell} 
\usepackage{soul} 
\usepackage{braket} 
\usepackage{cleveref} 
\Crefname{equation}{Eq.}{Eqs.}
\Crefname{figure}{Fig.}{Figs.}
\Crefname{section}{Sec.}{Secs.}
\usepackage{xcolor}
\definecolor{mygreen}{rgb}{0.0, 0.5, 0.0}
\usepackage{textcomp} 
\usepackage{siunitx} 
\usepackage{paralist} 
\usepackage{wrapfig} 
\usepackage{booktabs} 
\usepackage{rotating}
\usepackage{comment}
\usepackage{todo}

\newcommand{\rev}[1]{{\color{black}{#1}}}

\journal{a journal}

\begin{document}

\begin{frontmatter}

\title{On the Robustness of Sparse Counterfactual Explanations \\to Adverse Perturbations}

\author{Marco Virgolin\corref{mycorrespondingauthor}}
\address{Centrum Wiskunde \& Informatica, Science Park 123, 1098 XG Amsterdam, the Netherlands}
\ead{marco.virgolin@cwi.nl}

\author{Saverio Fracaros}
\address{University of Trieste, via Weiss 2, 34128 Trieste, Italy}


\cortext[mycorrespondingauthor]{Corresponding author}


\begin{abstract}
Counterfactual explanations (CEs) are a powerful means for understanding how decisions made by algorithms can be changed.
Researchers have proposed a number of desiderata that CEs should meet to be practically useful, such as requiring minimal effort to enact, or complying with causal models.
We consider a further aspect to improve the usability of CEs: 
robustness to adverse perturbations, which may naturally happen due to unfortunate circumstances.
Since CEs typically prescribe a \emph{sparse} form of intervention (i.e., only a subset of the features should be changed), we \rev{study the effect of addressing robustness separately for the features that are recommended to be changed and those that are not.}
Our definitions are workable in that they can be incorporated as penalty terms in the loss functions that are used for discovering CEs.
To experiment with robustness, we create and release code where five data sets (commonly used in the field of fair and explainable machine learning) have been enriched with feature-specific annotations that can be used to sample meaningful perturbations.
Our experiments show that CEs are often not robust and, if adverse perturbations take place (even if not worst-case), the intervention they prescribe may require a much larger cost than anticipated, or even become impossible.
However, accounting for robustness in the search process, which can be done rather easily, allows discovering robust CEs systematically.
Robust CEs \rev{make additional intervention to contrast perturbations much less costly than non-robust CEs}.
\rev{We also find that robustness is easier to achieve for the features to change, posing an important point of consideration for the choice of what counterfactual explanation is best for the user.}
Our code is available at: \url{https://github.com/marcovirgolin/robust-counterfactuals}.
\end{abstract}


\end{frontmatter}


\section{Introduction}

Modern Artificial Intelligence (AI) systems often rely on machine learning models such as ensembles of decision trees and deep neural networks~\cite{friedman2001greedy,ke2017lightgbm,lecun2015deep}, 
which are \emph{massive} in terms of number of parameters.
Massive models are appealing because, under proper training and regularization regimes, they are often unmatched by smaller models~\cite{belkin2019reconciling,nakkiran2021deep}.
However, as massive models perform myriads of computations, it can be very difficult to interpret and predict their behavior.
Because of this, massive models are often called \emph{black-box models}, and ensuring that their use in high-stakes applications (e.g., of medicine and finance) is fair and responsible can be challenging~\cite{goodman2017european,jobin2019global}.

The field of eXplainable AI (XAI) studies methods to dissect and analyze black-box models~\cite{adadi2018peeking,guidotti2018survey} (as well as methods to generate interpretable models when possible~\cite{rudin2019stop}). 
Famous methods of XAI include feature relevance attribution~\cite{ribeiro2016should,lundberg2017unified}, explanation by analogy with prototypes~\cite{kim2016examples,chaofan2019this}, and, of focus in this work, \emph{counterfactual explanations}. 
Counterfactual explanations enable to reason by contrast rather than by analogy, as they show in what ways the input given to a black-box model needs to be changed for the model to make a different decision~\cite{wachter2017counterfactual,stepin2021survey}.
A classic example of counterfactual explanation is:
\emph{``Your loan request has been rejected. If your salary was \num{60000}\$  instead of \num{50000}\$ and your debt was \num{2500}\$ instead of \num{5000}\$, your request would have been approved.''}
A user who obtains an unfavourable decision can attempt to overturn it by intervening according to the counterfactual explanation.

Normally, the search of counterfactual explanations is formulated as an optimization problem (see \Cref{sec:problem-statement} for a formal description). 
Given the feature values that describe the user as starting point, we seek the minimal changes to those feature values that result in a point for which the black-box model makes a different (and oftentimes, a specific favourable) decision.
We wish the changes to be minimal for two reasons: 
one, to learn about the behavior of the black-box model \rev{for a neighborhood of data points, e.g., to assess its fairness (although this is not guaranteed in general, see e.g.,~\cite{slack2021counterfactual})}; 
two, in the hope that putting the counterfactual explanation into practice by means of real-life intervention will require minimal effort too.
For counterfactual explanations to be most useful, more desiderata than requiring minimal feature changes may need to be taken into account (see \Cref{sec:related})~\cite{barocas2020hidden}. 

In this paper, we consider a desideratum that can be very important for the usability of counterfactual explanations: \emph{robustness to adverse perturbations}.
By adverse perturbations we mean changes in feature values that happen due to unfortunate circumstances beyond the user's control, \rev{making reaching the desired outcome no longer possible, or requiring the user to put more effort than originally anticipated}.
These unfortunate circumstances can have various origins, e.g., time delays, measurement corrections, biological processes, and so on.
For example, if a counterfactual explanation for improving a patient's heart condition prescribes lowering the patient's blood pressure, the chosen treatment may need to be employed for longer, or even turn out to be futile, if the patient has a genetic predisposition to resist that treatment (for more examples, see \Cref{sec:datasets} and choices made in the coding of our experiments, in \texttt{robust\_cfe/dataproc.py}).

We show that, if adverse perturbations might happen, one can and \emph{should} seek counterfactual explanations that are robust to such perturbations.
A particular novelty of our work is that we distinguish between whether perturbations impact the features that counterfactual explanations prescribe to \emph{change} or \emph{keep as they are} \rev{(note that some features may be irrelevant and can be changed differently than how prescribed by a counterfactual explanation, we address this in \Cref{sec:sparsity})}.
This is because counterfactual explanations are normally required to be \emph{sparse} in terms of the intervention they prescribe (i.e., only a subset of the features should be changed), for better usability (see \Cref{sec:problem-statement}).
As it will be shown, making this discrimination allows to improve the effectiveness and efficiency with which robustness can be accounted for.
\rev{Consequently, one might need to consider carefully which counterfactual explanation to pursue, based on whether they are robust to features to change or keep as they are.}

In summary, this paper makes the following contributions:
\begin{enumerate}
    \item We propose two workable definitions of robustness of counterfactual explanations that concern, respectively, the features prescribed to be changed and those to be kept as they are;
    \item We release code to support further investigations, where five existing data sets are annotated with perturbations and plausibility constraints that are tailored to the features and type of user seeking recourse;
    \item We provide experimental evidence that accounting for robustness is important to prevent adverse perturbations from making it very hard or impossible to achieve recourse through counterfactual explanations, when adverse perturbations are sampled from a distribution (i.e., they are not necessarily worst-case ones);
    \item \rev{We show that robustness for the features to change is far more reliable and computationally efficient to account for than robustness for the features to keep as they are;}
    \item \rev{Additionally, we propose a simple but effective genetic algorithm that outperforms several existing gradient-free search algorithms for the discovery of counterfactual explanations.
    The algorithm supports plausibility constraints and implements the proposed definitions of robustness.}
\end{enumerate}



\section{Preliminaries}
\label{sec:problem-statement}
\rev{Let us assume we are given a point $\mathbf{x} = \left( x_1, \dots, x_d \right)$, where $d$ is the number of features.
Each feature takes values either in (a subset of) $\mathbb{R}$, in which case we call it a \emph{numerical} feature, or in (a subset of) $\mathbb{N}$, in which case we call it a \emph{categorical} feature.
For categorical features, we use natural numbers as a convenient way to identify their categories, but disregard ordering.
For example, for the categorical feature \emph{gender}, $0$ might mean \emph{male}, $1$ might mean \emph{female}, and $2$ might mean \emph{non-binary}.
Thus, $\mathbf{x} \in \mathbb{R}^{d_1} \times \mathbb{N}^{d_2}$, where $d_1+d_2=d$.
}

A \emph{counterfactual example}\footnote{\rev{Many authors use $\mathbf{x}^\prime$ to represent a counterfactual example for $\mathbf{x}$, instead of $\mathbf{z}$.
We chose $\mathbf{z}$ not to overload the notation with superscripts later on in the manuscript, for readability.
}} for a point $\mathbf{x}$ is a point $\mathbf{z} \in \mathbb{R}^{d_1} \times \mathbb{N}^{d_2}$ such that, given a classification (black-box) machine learning model $f : \mathbb{R}^{d_1} \times \mathbb{N}^{d_2} \rightarrow \{ c_1, c_2, \dots \}$ ($c_i$ is a decision or \emph{class}), $f(\mathbf{z}) \neq f(\mathbf{x})$.
We wish $\mathbf{z}$ to be \emph{close} to $\mathbf{x}$ under a meaningful distance function \rev{$\delta$} that is problem-specific and meets several desiderata (see \Cref{sec:related}).
For example, commonly-used distances that are capable of handling both numerical and categorical features are variants of Gower's distance~\cite{gower1971general} (see~\Cref{eq:loss-function} and, e.g., \cite{guidotti2018local} for a variant thereof). Often, when dealing with more than two classes, we also impose \rev{$f(\mathbf{z}) = t$, i.e., the \emph{target} class we desire $\mathbf{z}$ to be}. 
Other times, we wish to find a \emph{set} of counterfactual examples $\{ \mathbf{z}_1, \dots, \mathbf{z}_k \}$, possibly of different classes, to obtain multiple means of recourse or simply gain information on the decision boundary of $f$ nearby $\mathbf{x}$ (e.g., to explain $f$'s local behavior)~\cite{wachter2017counterfactual,sharma2020certifai,mothilal2020dice}.

For the sake of readability, we provide formal definitions only for the case \rev{when} all features are numerical \rev{(i.e., $\mathbf{x} \in \mathbb{R}^d, d_1=d, d_2=0$)}.
For completeness, we include explanations of how to deal with categorical features in the running text.
Furthermore, \rev{we assume feature independence.
While this assumption is rarely entirely met in real-world practice, it is commonly done in literature due to the lack of causal models (e.g., only four works consider causality in \Cref{sec:related}), and allows us to greatly simplify the introduction of the concepts hereby presented. 
We discuss the limitations that arise from this assumption in \Cref{sec:discussion}.}

A counterfactual \emph{explanation} is represented by a description of how $\mathbf{x}$ needs to be changed to obtain $\mathbf{z}$.
In other words, a counterfactual explanation is a prescription on what interventions should be made to \emph{reach} the respective counterfactual example.
For example, under the assumption of independence and all-numerical features, the difference $\mathbf{z} - \mathbf{x}$ is typically considered to be the counterfactual explanation for how to reach $\mathbf{z}$ from $\mathbf{x}$.
What particular form counterfactual explanations take is not crucial to our discourse, and we will use $\mathbf{z} - \mathbf{x}$ for simplicity.

We proceed by considering the following traditional setting \rev{where, for simplicity of exposition and without loss of generality, we will assume that features are pre-processed so that a difference in one unit in terms of feature $i$ is equivalent to a difference in one unit in feature $j$ (i.e., the user's effort is commensurate across different features). 
Alternatively, one can account for this in the computation of the distance (see, e.g., \Cref{eq:loss-function}).
We
} seek the (explanation relative to \rev{an}) \emph{optimal} $\mathbf{z}^\star$ with:
\begin{equation}\label{eq:traditional}
\begin{aligned} 
    \mathbf{z}^\star & \in \text{argmin}_\mathbf{z} \delta (\mathbf{z}, \mathbf{x}) \\
    \textit{with\ \ } & \delta(\mathbf{z}, \mathbf{x}) := || \mathbf{z} - \mathbf{x} ||_1  + \lambda || \mathbf{z} - \mathbf{x} ||_0 \\
    \textit{and subject to\ \ } & f(\mathbf{z}) = t \text{\ \ and\ \ } \mathbf{z} - \mathbf{x} \in \mathcal{P}.
\end{aligned}
\end{equation}
In other words, $\delta$ is a linear combination, weighed by $\lambda$, of the sum of absolute distances between the feature values of $\mathbf{x}$ and $\mathbf{z}$, and the count of feature values that are different between $\mathbf{x}$ and $\mathbf{z}$.
\rev{Note that $\mathbf{z}^\star$ needs not be unique, i.e., multiple optima may exist.}
Moreover, the difference $\mathbf{z} - \mathbf{x}$ must abide to some plausibility constraints specified in a collection $\mathcal{P}$.
\rev{We model plausibility constraints as a set of specifications, each relative to a feature $i$, concerning whether $z_i - x_i$ is allowed to be  $>0$, $<0$, and $\neq 0$, i.e., a feature can increase, decrease, or change at all (for categorical features, we only consider the latter).}
For example for a private individual who wishes to be granted a loan, one of such constraints may specify that they cannot reasonably intervene to change the \rev{value of a currency (such a feature is called \emph{mutable but not actionable}), i.e., counterfactual explanations must have $z_i - x_i = 0$, for $i$ representing currency value.
Similarly, the individual's age may increase but not decrease, i.e., $z_i - x_i > 0$, for $i$ representing age.}

We particularly consider the $L1$-norm (i.e., the term $||\cdot||_1$ of $\delta$ in \Cref{eq:traditional}) because it is reasonable to think that, for independent features, the total \emph{cost} of intervention (i.e., the effort the user must put) is the sum of the costs of intervention for each feature separately, and that these costs grow linearly.
Some works (e.g.,~\cite{guidotti2018local,laugel2018comparison}) choose the $L2$-norm ($||\cdot||_2$, also known as Euclidean norm) instead of the $L1$-norm; the definitions of robustness given in this paper can be easily adapted for the $L2$-norm.
Regarding the $L0$-norm (i.e., the term $||\cdot||_0$ of $\delta$ in \Cref{eq:traditional}), this term explicitly promotes a form of sparsity, as it seeks to minimize how many features have a different value between $\mathbf{z}$ and $\mathbf{x}$.
This is desirable because, oftentimes, the user can only reasonably focus on, and intervene upon, a limited number of features (even if this amounts to a larger total cost in terms of $L1$ compared to intervention on \emph{all} the features)~\cite{keane2020good}.

\section{\rev{Perturbations \& robustness}}
\label{sec:perturbations}
\rev{Unfortunate circumstances (e.g., inflation) might lead to more or different intervention to be needed, compared to what was originally prescribed by a counterfactual explanation (e.g., increase savings by \num{1000} \$ to be granted credit access).
Thus, instead of reaching $\mathbf{z}$ as intended by the counterfactual explanation, a different point $\mathbf{z}^\prime$ is obtained.
Note that while the effects of unfortunate circumstances can impact feature values, the circumstances themselves need not be encoded as feature values. 
In fact, we will only focus on the extent by which feature values may be perturbed by such circumstances.

Let us define the vector $\mathbf{w} = \mathbf{z} - \mathbf{z}^\prime$ as a \emph{perturbation} for the counterfactual example $\mathbf{z}$.
We assume that perturbations that afflict feature $i$ are sampled from some distribution $P(w_i)$ and we are interested in controlling for, or \emph{being robust to}, large magnitude perturbations that have reasonable risk. 
For example for normally-distributed perturbations, we might want to consider the values that can be sampled at the $95^\textit{th}$ or $99^\textit{th}$ percentile.
We will therefore assume that we can define a vector $\mathbf{p}=\left(p^{\{-\}}_1, p^{\{+\}}_1, \dots, p^{\{-\}}_{d}, p^{\{+\}}_{d} \right)$ where $p^{\{-\}}_{i} \leq 0$ and $p^{\{+\}}_{i} \geq 0$ represent, respectively, the smallest negative and largest positive perturbations that can reasonably happen to the $i^\textit{th}$ feature.
For example, if the $i^\textit{th}$ feature represents the blood pressure of a patient, then $p^{\{-\}}_i$ tells by how much the blood pressure might lower at most (e.g., as a consequence of dehydration) and $p^{\{+\}}_1$ tells by how much the blood pressure might raise at most (e.g., as a consequence of anti-inflammatory drug intake). 
Clinicians may be able to define this information from their experience or retrieve it from medical literature.
In general, the magnitudes of $p^{\{-\}}_i, p^{\{+\}}_i$ need not be the same, i.e., $|p^{\{-\}}_i| \neq |p^{\{+\}}_i|$.
}
Note that for an $i^\textit{th}$ feature that is categorical, decreases or increases $p^{\{-\}}_i, p^{\{+\}}_i$ as explained for numerical features are no longer meaningful. 
For categorical features, we will assume that $\mathbf{p}$ contains elements that represent what categorical perturbations are possible for that feature, i.e., $p_i$ will be a set of indices that represent categories.

\rev{Under the problem setting we considered in \Cref{sec:problem-statement}, perturbations that may afflict a counterfactual explanation define a \emph{box} (hyper-rectangle) of all possible points $\mathbf{z}^\prime$ that can be reached from $\mathbf{z}$ due to perturbations.
An example is illustrated in  \Cref{fig:first-example}.
We define the concept of \emph{$\mathbf{p}$-neighborhood} of $\mathbf{z}$ as follows:

\begin{definition}
\label{def:p-neighborhood} 
\emph{($\mathbf{p}$-neighborhood and $\mathbf{p}$-neighbors of a counterfactual example)} Given a model $f$, a point $\mathbf{x}$, a respective counterfactual example $\mathbf{z}$, and a vector of possible perturbations $\mathbf{p}$, the $\mathbf{p}$-neighborhood of $\mathbf{z}$ is the set: 
\begin{equation}\label{eq:p-neighborhood}
  N := \left\{ \mathbf{z}^\prime \mid z^\prime_i \in [z_i + p^{\{-\}}_i, z_i + p^{\{+\}}_i] \right\}.
\end{equation}
A point $\mathbf{z}^\prime \in N \text{ such that } \mathbf{z}^\prime \neq \mathbf{z}$ is called a $\mathbf{p}$-neighbor of $\mathbf{z}$.
\end{definition}

}

\begin{figure}
    \centering
    \includegraphics[width=0.5\linewidth]{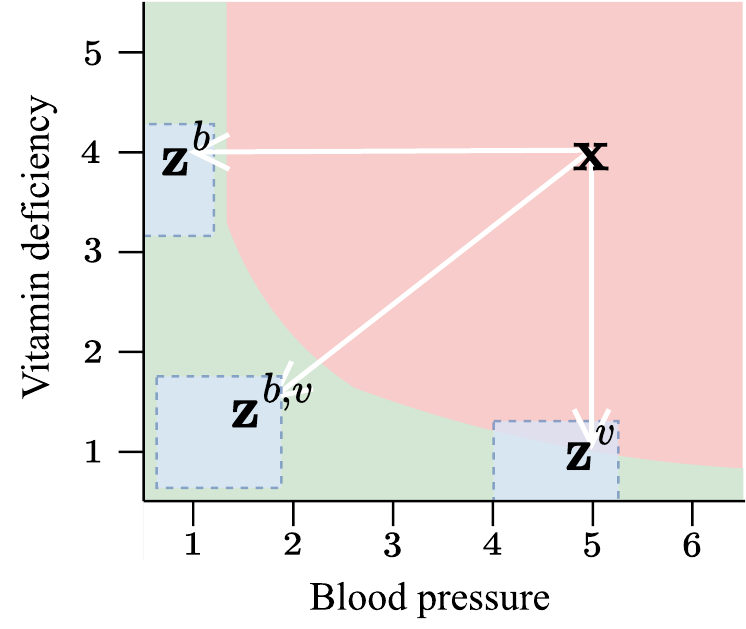}
    \caption{Example of considering robustness to perturbations when seeking counterfactual examples.
    The {\color{red}red} and {\color{mygreen}green} areas respectively represent \emph{high risk} and \emph{low risk} classifications of a cardiac condition according to a model $f$.
    The patient, represented by $\mathbf{x}$, is at high risk.
    Three possible counterfactual examples are shown for different interventions (white arrows): $\mathbf{z}^b$ for treating blood pressure, $\mathbf{z}^v$ for treating vitamin deficiency, and $\mathbf{z}^{b,v}$ for treating both.
    We assume to know the maximal extent of perturbation (under reasonable risk) for blood pressure and vitamin deficiency due to natural physiological events.
    This allows us to define the {\color{blue}blue} areas surrounding each counterfactual example.
    Perturbations $\mathbf{w}$ to one of the counterfactual examples can lead to any other point in the blue area.
    $\mathbf{z}^{v}$ is the best of the three in terms of proximity to $\mathbf{x}$ but its blue area partly overlaps with the red area.
    This means that there exist $\mathbf{w}$ such that $\mathbf{z}^{v} + \mathbf{w}$ leads to a point in the red area, invalidating the counterfactual explanation.
    In such cases, it is important to estimate if further intervention is possible so that can still be reached $\mathbf{z}^{v}$, and at what cost.
    }
    \label{fig:first-example}
\end{figure}

\rev{Not all perturbations are problematic. 
Our goal is to study robustness to \emph{adverse} perturbations, i.e., those for which $f(\mathbf{z} + \mathbf{w}) = f(\mathbf{z}^\prime) \neq t$.
In other words, we wish to seek counterfactual examples $\mathbf{z}$ that have no (or the fewest possible) $\mathbf{p}$-neighbors for which perturbations can cause the classification performed by $f$ to be different from $t$.
When that happens, we say that the counterfactual explanation has been \emph{invalidated} by the perturbation. 
However, it may be the case that invalidation is not \emph{permanent}: there may still exist intervention (i.e., a new counterfactual explanation) that adheres to the constraints in $\mathcal{P}$ and allows to overcome invalidation.
Therefore, in this work, we will seek to discover counterfactual explanations that are \emph{robust} in the sense that (i) if invalidated, further intervention remains possible, (ii) the cost of further intervention is small.
}

\rev{Unfortunately, if $f$ is assume to be a general model (e.g., not necessarily a linear one), then the following argument holds.

\begin{proposition}
\label{prop:general-f}
For a general $f$, information on the classification of a $\mathbf{p}$-neighbor (e.g., that $f(\mathbf{z}^\prime) = f(\mathbf{z})$ for $\mathbf{z}^\prime$ on the boundary of $N$) provides no information about the classification of another $\mathbf{p}$-neighbor (e.g., that $f(\mathbf{z}^{\prime}) \neq f(\mathbf{z}^{\prime\prime})$ for $\mathbf{z}^{\prime\prime}$ in the interior of $N$). 
\end{proposition}
}

\begin{proof}
We cannot preclude that the model $f$ is, for example, a neural network.
Under the universal approximation theorem~\cite{hornik1989multilayer}, $f$ may represent any function. 
Thus, $f$ may represent a \emph{Swiss cheese}-like function, where for example $f(\mathbf{z}^\prime) \neq f(\mathbf{z}^{\prime\prime})$ with $\mathbf{z}^{\prime\prime}:= \mathbf{z}^{\prime} + \mathbf{e}$ and $\mathbf{e} = \left( \varepsilon_1, \dots, \varepsilon_d \right)$ different from the zero-vector, however small $|\varepsilon_i|, \forall i$.
\end{proof}

This proposition means that if no information on, e.g., regularity or smoothness of $f$ is available, then we must check each and every $\mathbf{p}$-neighbor $\mathbf{z^\prime}$ of $\mathbf{z}$ to assess whether some of them may invalidate the explanation, i.e., $\exists \mathbf{z}^\prime$ such that  $f(\mathbf{z}^\prime) \neq t$.
Checking all neighbors is typically not feasible, e.g., as soon as some of the features are real valued. 
Thus, the best one can do is to take an approximate approach.
For example, a Monte-Carlo sampling approach can be used where a batch of random points within $N$ is considered, hoping that the batch is representative of all points in $N$.
As we will show in the next sections, a better strategy can be designed if sparsity is considered.

\rev{We conclude this section by noting that perturbations, as described so far, are \emph{absolute}, i.e., independent of the starting point $\mathbf{x}$, the counterfactual in consideration $\mathbf{z}$, or the intervention entailed by the counterfactual explanation $\mathbf{z}-\mathbf{x}$.
Perturbations to feature $i$ might however depend on $x_i$ and $z_i$, i.e., be sampled from a distribution $P(w_i | x_i, z_i)$.
For example, due to market fluctuations, a return on investment may be smaller than anticipated by $5\%$ of the expected value.
Such type of \emph{relative} perturbations entail different $\mathbf{p}$-neighborhoods for different $\mathbf{x}$ and $\mathbf{z}$.
For simplicity and without loss of generality,} we will proceed by assuming that perturbations can only be absolute. 
We explain how we also included \emph{relative} perturbations in the annotations used for our experiments in \Cref{sec:experimental-setup}.

\rev{
\section{Sparsity, features in $\mathcal{C}$ and $\mathcal{K}$}
\label{sec:sparsity}
}
We use the form of sparsity mentioned \rev{in \Cref{sec:problem-statement}} to partition the features into two sets.
\rev{As mentioned before, sparsity is an important desideratum because it may not be reasonable to expect that the user can realistically intervene on, and keep track of, \emph{all} the features to achieve recourse.}
\rev{Given a specific counterfactual explanation $\mathbf{z}$ for the point $\mathbf{x}$,} we call the set containing the (indices of the) features whose values should \emph{change}  
$\mathcal{C} = \{i \in \{1, \dots, d\} \mid z_i \neq x_i \}$, and its complement, i.e., the set of the (indices of the) features whose values should be \emph{kept as they are},
$\mathcal{K} = \{i \in \{1, \dots, d\} \mid z_i = x_i \}$.
Typically, because a sufficiently large $\lambda$ is used, or because of \rev{the plausibility constraints specified in} $\mathcal{P}$, $\mathcal{K} \neq \emptyset$. 

\rev{Note that the proposed partitioning between $\mathcal{C}$ and $\mathcal{K}$ implicitly assumes that all features are relevant to the counterfactual explanation. 
If certain features are irrelevant, perturbations to those features will have no effect on $f$'s decision, and thus those features need not be accounted for when assessing robustness.
This means that accounting for irrelevant features makes assessing robustness more computationally expensive than needed.
However, as $f$ is considered to be a black-box, we cautiously assume that all features are relevant for assessing robustness.
}

\rev{We will proceed by accounting for perturbations and respective robustness separately for features in $\mathcal{C}$ and $\mathcal{K}$}. 
\rev{Accounting for robustness separately is important because, as we will show, assessing robustness for features in $\mathcal{C}$ can be done far more efficiently and be more effective than for features in $\mathcal{K}$.
Knowing this, if multiple counterfactual explanations can be found, the user may want to choose the counterfactual explanation that fits him/her best based on the robustness it exhibits in terms of $\mathcal{C}$ and $\mathcal{K}$.
}
In the next section, we present our first notion of robustness, which concerns $\mathcal{C}$.

\section{Robustness for $\mathcal{C}$}
\label{sec:c-robust-cfe}
We begin by focusing on the features that the counterfactual explanation instructs to change, i.e., the features (whose indices are) in $\mathcal{C}$.
\rev{Recall that we assume that a vector of maximal perturbation magnitudes $\mathbf{p}$ can be defined. 
This leads us to the following definition.}

\begin{definition}
\label{def:c-perturbation}
\rev{\emph{($\mathcal{C}$-perturbation)}
Given a point $\mathbf{x}$, a respective counterfactual example $\mathbf{z}$, and the vector of maximal magnitude perturbations $\mathbf{p}$, a \emph{$\mathcal{C}$-perturbation} for the counterfactual explanation $\mathbf{z}- \mathbf{x}$ is a vector
\begin{align}
    & \mathbf{w}^c = (w^c_1, \dots, w^c_d), \text{\ where\ \ } \\
& 
\begin{cases}
    p^{\{-\}}_{i} \leq w^c_i \leq p^{\{+\}}_{i}   & \text{if } i \in \mathcal{C}, \\
    w^c_i = 0  & \text{otherwise, i.e., }i \notin \mathcal{C}.\\
\end{cases}
\end{align}
and such that $\exists i : w^c_i > 0$, i.e., $\mathbf{w}^c$ is not the zero-vector. 
}
\end{definition}
\rev{In other words, a $\mathcal{C}$-perturbation is a perturbation that acts only on features in $\mathcal{C}$, and at least on one of such features.
Next, we use the concept of $\mathcal{C}$-perturbations to introduce the one of $\mathcal{C}$-setbacks.

\begin{definition}
\label{def:c-setback}
\rev{\emph{($\mathcal{C}$-setback)}
A \emph{$\mathcal{C}$-setback} for the counterfactual explanation $\mathbf{z}- \mathbf{x}$ is a $\mathcal{C}$-perturbation such that
\begin{equation}
\begin{cases}
    p^{\{-\}}_{i} \leq w^c_i \leq 0  & \text{if } z_i - x_i > 0\\
    0 \leq w^c_i \leq p^{\{+\}}_{i}  & \text{if } z_i - x_i < 0\\
    w^c_i = 0  & \text{otherwise, i.e., }i \notin \mathcal{C}.\\
\end{cases}
\end{equation}
}
We denote $\mathcal{C}$-setbacks with $\mathbf{w}^{c,s}$.
\end{definition}
}
\rev{In words, a $\mathcal{C}$-setback is a $\mathcal{C}$-perturbation where each and every element of the perturbation $w^{c,s}_i$ is of opposite sign to the counterfactual explanation $z_i - x_i$.
We can interpret the meaning of $\mathcal{C}$-setbacks $\mathbf{w}^{c,s}$ as vectors that \emph{push the user away from $\mathbf{z}$ and back towards $\mathbf{x}$} along the direction of intervention.
Furthermore, we call a \emph{maximal $\mathcal{C}$-setback}, denoted by $\mathbf{w}^{c,s}_\text{max}$, $\mathcal{C}$-setback whose elements that correspond to features in $\mathcal{C}$ have maximal magnitude, i.e., $w^{c,s}_i = p^{-}_{i}$ if $z_i - x_i > 0$ and  $w^{c,s}_i = p^{+}_{i}$ if $z_i - x_i < 0$.
An example is given in \Cref{fig:robustness-c}.
}

\rev{
\begin{figure}
    \centering
    \includegraphics[width=.5\linewidth]{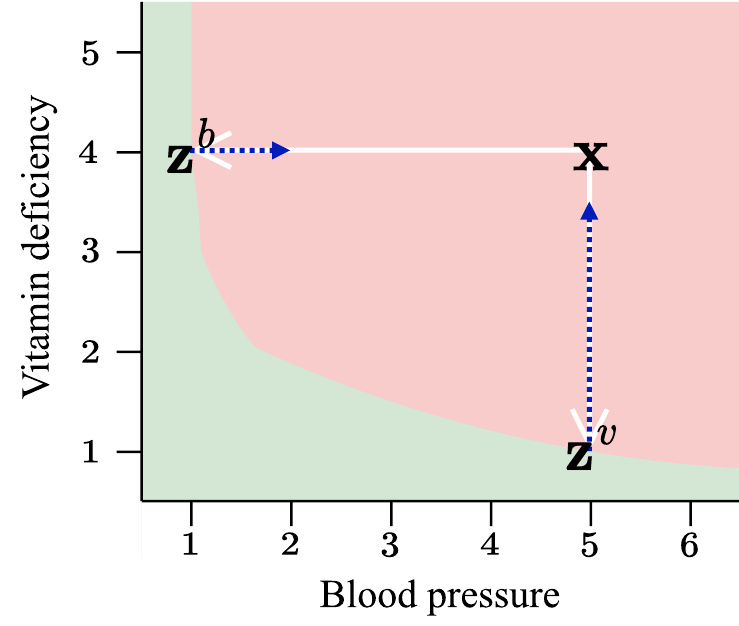} 
    \caption{\rev{Example of $\mathcal{C}$-setbacks.
    The {\color{red}red} and {\color{mygreen}green} areas represent \emph{high risk} and \emph{low risk} classifications of a cardiac condition according to a model $f$.
    The patient, represented by $\mathbf{x}$, is at high risk.
    Two treatments are possible but cannot be administered jointly due to drug incompatibility, hence sparsity of intervention is needed (only one of the two treatments can be pursued).
    The closest (and thus optimal) counterfactual example is $\mathbf{z}^v$ and concerns treating vitamin deficiency (white arrow pointing down).
    Another counterfactual example is $\mathbf{z}^b$ and concerns treating blood pressure (white arrow pointing left).
    Maximal $\mathcal{C}$-setbacks are shown for both counterfactual examples (blue dashed segments).
    The setbacks can make both counterfactual examples invalid, however further intervention (treatment administration) allows to reach them. 
    If maximal $\mathcal{C}$-setbacks are accounted for, $\mathbf{z}^b$ becomes more costly to pursue than $\mathbf{z}^v$ (cost of $4+1$ vs.~$3+2.5$, respectively).
    }
    }
    \label{fig:robustness-c}
\end{figure}
}

\begin{wrapfigure}{r}{0.5\textwidth}
    \centering
    \includegraphics[width=\linewidth]{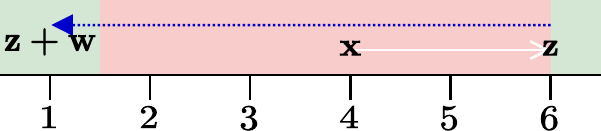}
    \caption{\rev{One-dimensional example of how a $\mathcal{C}$-setback can be advantageous if the magnitude of perturbation exceeds the magnitude of intervention.
    }}
    \label{fig:example-degenerate-setback}
\end{wrapfigure}
$\mathcal{C}$-setbacks are arguably more interesting than $\mathcal{C}$-perturbations because $\mathcal{C}$-setbacks are the subset of these perturbations that plays against the user.
In fact, certain $\mathcal{C}$-perturbations might be advantageous, enabling to reach $\mathbf{z}$ with less intervention than originally provisioned (i.e., when the sign of $w^c_i$ and that of $z_i - x_i$ matches).
To account for robustness, we are interested in understanding whether perturbations can prevent us to reach $\mathbf{z}$, hence we will proceed by focusing exclusively on $\mathcal{C}$-setbacks.

It is important to note that even $\mathcal{C}$-setbacks can be advantageous if one allows their perturbations to be of larger magnitude than intervention, i.e., if $|w^{c,s}_i| > |z_i - x_i|$ is allowed \rev{(\cite{fokkema2022attribution} discuss this aspect in detail)}. 
In a nutshell, if $|w^{c,s}_i| > |z_i - x_i|$, then a $\mathcal{C}$-setback can lead to a point that ``precedes'' $\mathbf{x}$ in terms of the direction of intervention. 
For that point, the intervention may be less costly than the one that was originally planned or entirely not needed because the point is of the target class (see, e.g., \Cref{fig:example-degenerate-setback}).
Advantageous situations are not interesting for robustness and counterfactual explanations that can be overturned by perturbations may well not be interesting to pursue.
\rev{We therefore consider any $\mathcal{C}$-setback to have elements capped by $|w^{c,s}_i| \leq |z_i - x_i|$.}

Perhaps the most interesting scenario for considering $\mathcal{C}$-setbacks is when dealing with $\mathbf{z}^\star$, since a counterfactual example that is optimal \rev{(i.e., one minimizes \Cref{eq:traditional})} is an ideal outcome.
The following simple result holds for $\mathbf{z}^\star$:
\begin{proposition}
\label{prop:optimal-c}
For any $\mathcal{C}$-setback $\mathbf{w}^{c,s}$ of $\mathbf{z}^\star$ \rev{(such that $|w^{c,s}_i| \leq |z^\star_i - x_i|$ for all $i$)}, $f(\mathbf{z}^\star + \mathbf{w}^{c,s}) \neq t$.
\end{proposition}

\begin{proof}
We use \emph{reduction ad absurdum}. 
Let us assume the opposite of what was said in \Cref{prop:optimal-c}, i.e., there exists $\mathbf{w}^{c,s}$ such that $f(\mathbf{z}^\star + \mathbf{w}^{c,s}) = t$.
Let $\mathbf{z}^\prime := \mathbf{z}^\star + \mathbf{w}^{c,s}$, and so $f(\mathbf{z}^\prime) = t$.
By construction of $\mathbf{w}^{c,s}$, $\delta(\mathbf{z}^\prime, \mathbf{x}) = \delta(\mathbf{z}^\star + \mathbf{w}^{c,s}, \mathbf{x}) < \delta(\mathbf{z}^\star, \mathbf{x})$.
In other words, $\mathbf{z}^\prime$ is of the target class and is closer to $\mathbf{x}$ than $\mathbf{z}^\star$ is.
This contradicts the fact that $\mathbf{z}^\star$ is optimal.
\end{proof}

Now, because of~\Cref{prop:optimal-c}, we are \emph{guaranteed} that if a $\mathcal{C}$-setback $\mathbf{w}^{c,s}$ happens to $\mathbf{z}^\star$, the resulting point will no longer be classified as $t$.
\rev{Intuitively, this is a natural consequence of the fact that optimal counterfactual examples lay on the border of the decision boundary as otherwise they would not be optimal.
Also, since $\mathbf{z}^\star$ is optimal, the respective $L0$ component for the distance between $\mathbf{x}$ and $\mathbf{z}^\star$ is minimal, i.e., all features in $\mathcal{C}$ and thus in $\mathbf{w}^{c,s}$ are relevant for the classification.
Given the premises just made, it becomes important to understand whether invalidation to $\mathbf{z}^\star$ can be averted with further intervention and, if so, whether the cost of such intervention can be minimized.
}

\rev{
We proceed by noting that, importantly, invalidation of a counterfactual explanation by a $\mathcal{C}$-setback can always be averted, i.e., further intervention to reach the intended $z_i$ for all $i \in \mathcal{C}$ is always  possible.
To see this, consider the fact that the intervention entailed by the counterfactual explanation $\mathbf{z}-\mathbf{x}$ must adhere to the plausibility constraints specified in $\mathcal{P}$ (else, $\mathbf{z}-\mathbf{x}$ would not be a possible counterfactual explanation).
Since $\mathcal{C}$-setbacks are aligned with the direction of the original intervention, the point $\mathbf{z}+\mathbf{w}^{c,s}-\mathbf{x}$, which is in between $\mathbf{x}$ and $\mathbf{z}$, must meet $\mathcal{P}$. 
It therefore suffice to apply further intervention along the originally-intended direction to recover the desired counterfactual example.
Under the $L1$-norm (as per the choice of $\delta$ in \Cref{eq:traditional}), the cost associated with the additional intervention needed to overcome a $\mathcal{C}$-setback $\mathbf{w}^{c,s}$ is simply $||\mathbf{w}^{c,s}||_1$.

Finally, under the reasonable assumption that the user keeps track of how the value of $x_i$ changes for $i\in \mathcal{C}$ over the course of intervention (otherwise, (s)he would not know when to stop and realize the counterfactual explanation), we can use \Cref{prop:optimal-c} in order to seek counterfactual examples that are optimal (i.e., require minimum intervention cost) under maximal $\mathcal{C}$-setbacks $\mathbf{w}^{c,s}_\text{max}$.
In the following definition, to highlight that  $\mathcal{C}$-setbacks depend on the specific $\mathbf{z}$ and $\mathbf{x}$ (as they determine $\mathcal{C}$) and avoid confusion, we use the function notation $W^{c,s}_\text{max}(\mathbf{z},\mathbf{x})$ in place of $\mathbf{w}^{c,s}_\text{max}$.
}

\begin{definition}\label{def:c-robust-cfe}
\emph{(Optimal counterfactual example under $\mathcal{C}$-setbacks)} 
Given a model $f$, a point $\mathbf{x}$, and a vector $\mathbf{p}$, 
\rev{we call a point $\mathbf{z}^{\star,c}$ such that
\begin{equation}
\label{eq:opt-with-worst-case-c}
\mathbf{z}^{\star,c} - W^{c,s}_\text{max}(\mathbf{z}^{\star,c},\mathbf{x}) \in \text{argmin}_{\left(\mathbf{z}-W^{c,s}_\text{max}(\mathbf{z}, \mathbf{x})\right)} \delta \left(  \mathbf{z} - W^{c,s}_\text{max}(\mathbf{z}, \mathbf{x}), \mathbf{x} \right),
\end{equation}
\emph{an optimal counterfactual example under $\mathcal{C}$-setbacks}.
}
\end{definition}

This definition gives us a way to \rev{seek a (multiple may exist) counterfactual explanation that entails minimal intervention cost when accounting for maximal $\mathcal{C}$-setbacks.}
\rev{Indeed, it suffices to equip a given search algorithm with \Cref{eq:opt-with-worst-case-c}, i.e., perform the following steps: (1) for any $\mathbf{z}$ to be evaluated, compute the respective $\mathbf{w}^{c,s}_\text{max}$, (2) instead of computing $\delta(\mathbf{z}, \mathbf{x})$, compute $\delta((\mathbf{z}-\mathbf{w}^{c,s}_\text{max}), \mathbf{x})$, and (3) at the end of the search, return the point that minimizes such distance, i.e., $\mathbf{z}^{\star,c}$.}

\rev{
Performing the computations just mentioned takes linear time in the number of features ($O(d)$) because we only need to build $\mathbf{w}^{c,s}_\text{max}$ (step 1 above) and subtract it from $\mathbf{z}$ prior to computing $\delta$ (step 2 above) for any given $\mathbf{z}$ ($f$ should still be evaluated on $\mathbf{z}$).
This is relatively fast (as demonstrated in \ref{sec:apdx-additional-runtime-robustness}), especially compared to the situation described in \Cref{sec:perturbations}, where one would need to use $f$ to predict the class of a number of neighbors of $\mathbf{z}$.
}
Note \rev{also} that in \Cref{eq:opt-with-worst-case-c} setbacks are subtracted from counterfactual examples when computing $\delta$, to account for the fact that the cost should increase (recall the construction of $\mathcal{C}$-setbacks in Def.~\ref{def:c-setback}).

\section{Robustness for $\mathcal{K}$}
\label{sec:k-robust-cfe}

We now consider $\mathcal{K}$, i.e., the set concerning the features that should be kept to their current value.
\rev{Mirroring the notion of $\mathcal{C}$-perturbation (\Cref{def:c-perturbation}), we can define a $\mathcal{K}$-perturbation to be a vector $\mathbf{w}^k$ such that $p^{\{-\}}_i \leq w^k_i \leq p^{\{+\}}_i$ if $i\in \mathcal{K}$ and $w^k_i = 0$ if $i \notin \mathcal{K}$.}
\rev{
Similarly, we can cast the concept of neighborhood from \Cref{def:p-neighborhood} to consider only $\mathcal{K}$-perturbations, leading to:
}

\begin{definition}
\label{def:k-neighborhood} 
\emph{($\mathcal{K}$-neighborhood and $\mathcal{K}$-neighbors of a counterfactual example)} Given a model $f$, a point $\mathbf{x}$, a respective counterfactual example $\mathbf{z}$, and a vector of possible perturbations $\mathbf{p}$, the $\mathcal{K}$-neighborhood of $\mathbf{z}$ under $\mathbf{p}$ is the set: 
\begin{equation}\label{eq:k-neighborhood}
  K := \left\{ \mathbf{z}^\prime \mid
  \begin{array}{@{}l@{}}
    z^\prime_i \in [z_i + p^{\{-\}}_i, z_i + p^{\{+\}}_i], \text{ if } i \in \mathcal{K}
    \\
    z^\prime_i  = z_i \text{ otherwise }
  \end{array} \, 
  \right\}.
\end{equation}
A point $\mathbf{z}^\prime \in K \text{ such that } \mathbf{z}^\prime \neq \mathbf{z}$ is called a $\mathcal{K}$-neighbor of $\mathbf{z}$.
\end{definition}

For a categorical feature $i \in \mathcal{K}$, the neighborhood can be built by swapping $z_i$ with one of the possibilities opportunely listed in $p_i \in \mathbf{p}$, where $p_i$ will be a set containing categories perturbations can lead to.

Next, we use $K$ to define the concept of \rev{vulnerability to  $\mathcal{K}$-perturbations}:
\begin{definition}\label{def:k-robust-cfe}
\emph{(\rev{Vulnerability to  $\mathcal{K}$-perturbations})}
Given a model $f$, a point $\mathbf{x}$, and a vector $\mathbf{p}$, a counterfactual example $\mathbf{z}$ is \rev{vulnerable to  $\mathcal{K}$-perturbations if} $\rev{\exists} \mathbf{z}^\prime \in N(\mathbf{z}, \mathbf{p})$ such that $f(\mathbf{z}^\prime) \neq f(\mathbf{z})$.
\end{definition}
Informally, this definition says that $\mathbf{z}$ is \rev{vulnerable to $\mathcal{K}$-perturbations} if the decision boundary surrounding $\mathbf{z}$ is \rev{not} sufficiently loose with respect to the features \rev{in $\mathcal{K}$}.
\rev{\Cref{fig:robustness-k} shows an example.}
\rev{The reason why vulnerability to $\mathcal{K}$-perturbations is particularly important is that, differently from the case of $\mathcal{C}$-perturbations, a $\mathcal{K}$-perturbation can invalidate the counterfactual explanation \emph{permanently}.
In fact, a $\mathcal{K}$-perturbation changes $\mathbf{z}$ along a different direction than the one of intervention.
Thus, a $\mathcal{K}$-perturbation can lead to a point $\mathbf{z}^\prime$ from which there exists no plausible intervention to reach the originally-intended $\mathbf{z}$ from.

For example, consider the feature $i$ to represent \emph{inflation} as a mutable but not actionable feature, i.e., a feature that can be changed (e.g., by global market trends) but not by the user.
$\mathcal{P}$ will state that no (user) intervention can exist to change $i$, i.e., $\mathcal{P}$ imposes $z_i - x_i = 0$.
However, an unfortunate circumstance such as a the financial crisis of 2008 may lead to a large inflation increase ($p^+_\textit{i} > 0$). 
Consequently, it may become impossible for the user to obtain the desired loan, e.g., because the bank does not hand out certain loans when the inflation is too high.
}

\begin{figure}
    \centering
    \includegraphics[width=.5\linewidth]{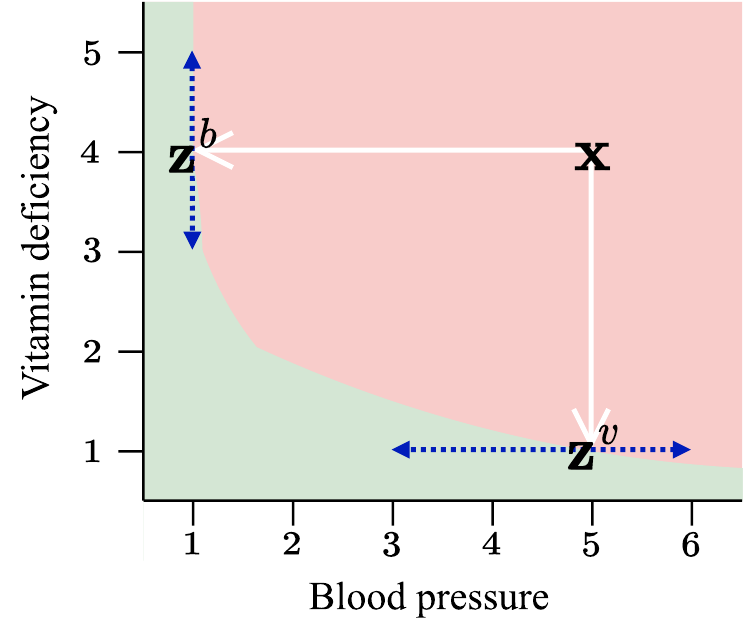} 
    \caption{\rev{Example of $\mathcal{K}$-perturbations.
    The counterfactual example $\mathbf{z}^v$ is vulnerable to $\mathcal{K}$-perturbations because these can lead to the red area; the same is not true for $\mathbf{z}^b$.
    If it is not plausible to reduce blood pressure, then $\mathcal{K}$-perturbations to $\mathbf{z}^v$ can lead to permanent invalidity.
    Else, they can be resolved with further intervention, in terms of blood pressure.
    }
    }
    \label{fig:robustness-k}
\end{figure}

\rev{Now,} recall that the reason why \Cref{def:c-robust-cfe} can be used for the case of $\mathcal{C}$-perturbations is that  \Cref{prop:optimal-c} holds, i.e., there cannot exist points of class $t$ between $\mathbf{x}$ and an optimal counterfactual example $\mathbf{z}^\star$.
\rev{The same does not hold for $\mathcal{K}$-perturbations, i.e., since the features in $\mathcal{K}$ are orthogonal to the direction of intervention, it can happen that the maximal perturbation to a feature $i \in \mathcal{K}$ leads to a point $\mathbf{z}^\prime$ for which $f(\mathbf{z}^\prime)=t$, while a non-maximal perturbation to the same feature can lead to a point $\mathbf{z}^{\prime\prime}$ for which $f(\mathbf{z}^{\prime\prime}) \neq t$.
Thus, checking for maximal perturbations is no longer sufficient: we must check instead for all points in the $\mathcal{K}$-neighborhood $K$.
}

\rev{As mentioned in \Cref{sec:perturbations}, checking each and every point in a neighborhood may not be feasible.}
Thus, we propose to approximate the assessment of how $\mathcal{K}$-robust (i.e., non-vulnerable to perturbations in $\mathcal{K}$) counterfactual explanations can be, with Monte-Carlo sampling.
Let $\mathbf{1}_{f(\mathbf{z})} : K \rightarrow \{0,1\}$ be the indicator function that returns $1$ for $\mathcal{K}$-neighbors that share the same class of $\mathbf{z}$ (i.e., $f(\mathbf{z})$), and $0$ for those that do not.
Taken a random sample of $m$ $\mathcal{K}$-neighbors, we define the following score:
\begin{equation}
\label{eq:practical-k-robust}
    \mathcal{K}\text{-robustness score}(\mathbf{z},m) = \frac{1}{m}
      \sum_{i=1}^{m} \mathbf{1}_{f(\mathbf{z})}(\mathbf{z}^\prime_i).
\end{equation}
\rev{We remark that even if} $\mathcal{K}\text{-robustness score}(\mathbf{z}, m) = 1$, we are not guaranteed that $\mathbf{z}$ is $\mathcal{K}\text{-robust}$\rev{, because the score is an approximation}.
Still, this score can be used to determine which counterfactual examples are preferable to pursue in that they are associated with a smaller risk that adverse perturbations will \rev{invalidate them (permanently or not)}.

\section{Experimental setup}
\label{sec:experimental-setup}
In this section, we \rev{firstly describe the preparation of the data sets used in our experiments. 
Secondly, we describe the search algorithms considered for finding near-optimal counterfactual explanations. 
Lastly, we describe the loss function considered, as well as how to incorporate the proposed notions of robustness into it.}

\subsection{Data sets} 
\label{sec:datasets}
\Cref{tab:datasets} summarizes the data sets we consider.
For each data set, we make an assumption on the type of user who seeks recourse, e.g., the user could be a private individual seeking to increase their income, or a company seeking to improve the productivity of its employees. 
\rev{Based on this, we manually define} the target class $t$\rev{,} the set of plausibility constraints $\mathcal{P}$ on what interventions are reasonably plausible\rev{, and the collection $\mathbf{p}$ of maximal magnitudes from which perturbations can be sampled (we will consider uniform and normal distributions)}.
We named the data sets in \Cref{tab:datasets} to represent their purpose.
Originally, \emph{Credit risk} (abbreviated to Cre) is known as \emph{South German Credit Data}~\cite{groemping2019south}, which is a recent update that corrects inconsistencies in the popular \emph{Statlog German Credit Data}~\cite{hofmann1994statlog}.
\emph{Income} (Inc) is often called \emph{Adult} or 
\emph{Census income}~\cite{kohavi1996adult,kohavi1996scaling}.
\emph{Housing price} (Hou) is also known as \emph{Boston housing}~\cite{harrison1978hedonic}
and is often used for research on fairness and interpretability because one of its features raises ethical concerns~\cite{carlisle2019racist}.
\emph{Productivity} (Pro) concerns the productivity levels of employees producing garments~\cite{imran2021mining}.
Lastly, \emph{Recidivism} (Rec) is a data set collected by an investigation of ProPublica about possible racial bias in the commercial software \emph{COMPAS}, which intends to estimate the risk that an inmate will re-offend~\cite{larson2016how}.
Examples of recent works in fair and explainable machine learning that adopted (some of) these data sets (each) are~\cite{guidotti2018local,kearns2018preventing,laugel2019dangers,ding2021retiring,lacava2020genetic,virgolin2021mlpie,dominguez2021adversarial}.

We pre-process the data sets similarly to how done often in the literature. 
This includes, e.g., removal of redundant features and of observations with missing values, and limiting the number of observations considered for Rec.
\rev{Regarding our annotations for the perturbations, } numerical features can have perturbations that increase or decrease the feature value, in absolute or relative terms; we compute relative perturbations with respect to $\mathbf{z}$.
For example, for the numerical feature \emph{capital-gain} of Inc, we assume that perturbations can happen that lead up to a relative $5\%$ increase or $10\%$ decrease of that feature, based on the value to achieve for that feature.
For categorical features, we define only absolute perturbations, i.e., possible changes of category are not conditioned to the current category.
The choices we made to build $\mathbf{p}$ are subjective, we elaborate on this in \Cref{sec:discussion}.
We sample the amount of perturbation using a uniform or normal distribution, as indicated in \rev{\Cref{sec:results-robustness}}.
\rev{\Cref{tab:examples-perturbations} shows some examples of maximal perturbations we annotated.}
As mentioned before, we also define plausibility constraints $\mathcal{P}$ for each data set.
Each constraint is specific to a feature.
For an $i^\textit{th}$ numerical feature, possible constraints are $z_i - x_i \geq 0$, $z_i - x_i \leq 0$, $z_i - x_i = 0$, and \emph{none}. 
For an $i^\textit{th}$ categorical feature, possible constraints are $z_i = x_i$ and \emph{none}.
Full details about our pre-processing and definition of $\mathbf{p}$ and $\mathcal{P}$ are documented in the form of comments in our code, in \texttt{robust\_cfe/dataproc.py}.

\begin{sidewaystable}[]
    \centering
    \caption{
        Considered data sets, where $n$ and $d$ (resp., $d_2$) indicate the number of observations and features (only categorical) after pre-processing. The column $t$ is the target class for the (simulated) user.
        Plausib.~constr.~reports the number of plausibility constraints that allow features to only increase ($\geq$), remain equal ($=$), and decrease ($\leq$).
        The column Perturb.~reports the number of perturbations concerning numerical (N) and categorical (C) features.
        Finally, $f$'s acc.~reports the average (across five folds) test accuracy of random forest.
    }
    \begin{tabular}[t]{l c c c c c c c l c}
        \toprule
        Data set (abbrev.) & {$n$} & {$d$} & {$d_2$} & Classes & User & {$t$} & {Plausib.~constr.} & {Perturb.} & {$f$'s acc.} \\
        \midrule
        Credit risk (Cre) & 1000 & 20 & 6 & \small High, low & \small Individual & \small Low & $\geq$:3, $=$:8, $\leq$:0 & N:6, C:0 &
        0.76\\
        \small Income (Inc) & 1883 & 12 & 7 & \small High, low & \small Individual & \small High & $\geq$:2, $=$:3, $\leq$:0 & N:4, C:4 & 0.83\\
        \small House price (Hou) & 506 & 13 & 1 & \small High, low & \small Municipality & \small Low & $\geq$:0, $=$:3, $\leq$:1 & N:11, C:0 & 0.93\\
        \small Productivity (Pro) & 1196 & 12 & 5 & \small High, med., low & \small Company & \small High & $\geq$:0, $=$:0, $\leq$:0 & N:5, C:2 & 0.79\\
        \small Recidivism risk (Rec) & 2000 & 10 & 6 & \small High, low & \small Inmate & \small Low & $\geq$:2, $=$:2, $\leq$:0 & N:3, C:2 & 0.80\\
        \bottomrule
    \end{tabular}
    \label{tab:datasets}
\end{sidewaystable}

\begin{table}[]
    \centering
    \rev{
    \resizebox{\columnwidth}{!}{%
    \begin{tabular}{lcccl}
        \toprule
        \multirow{2}{*}{D.set} & \multirow{2}{*}{Feature} & Decrease & Increase & \multirow{2}{*}{Note}\\
        & & \multicolumn{2}{c}{or Categories}\\
        \midrule
        Cre & 
        \emph{Savings} & $10\%$ & $10\%$  & \emph{Might happen to save less or more relative to what intended.} \\
        Inc & 
        \emph{Marital status} & \multicolumn{2}{c}{\{ single, married, widowed, \dots  \}} & \emph{Unforeseen change due to, e.g., proposal, divorce, death.} \\
        Hou & \emph{Crime rate} & $1\%$ & $5\%$ & \emph{Relative, might increase more than decrease.}\\
        Pro & \emph{Overtime} & 3 & 3 & \emph{Up to 3 more or less days of overtime might be needed.} \\
        Rec & \emph{Age} & 0 & 2 & \emph{Judicial system delays for up to 2 years.}\\
        \bottomrule
    \end{tabular}
    }
    \caption{Examples of perturbations that we manually annotated on the considered data sets.
    We take relative perturbations (those with $\%$) with respect to the value of the feature in the intended counterfactual example $\mathbf{z}$ in consideration by the search algorithm.}
    }
    \label{tab:examples-perturbations}
\end{table}

\rev{
\subsection{Black-box models}
}
We consider random forest \rev{and neural networks (with standard multi-layer perceptron architecture)} as black-box machine learning models $f$.
We use Scikit-learn’s implementations~\cite{pedregosa2011scikit}.
We \rev{assume that we can only access the predictions of} $f$\rev{, and no other information such as model parameters or gradients}.
Our experiments are repeated across a stratified five-fold cross-validation, and each model is obtained by grid-search hyper-parameter tuning.
\rev{Once trained, the models obtain test accuracy varying from $70\%$ to more than $90\%$ on average across the different data sets, i.e., meaningful decision boundaries are learned. See \ref{sec:apdx-hyperparams-random-forest} for details on hyper-parameter tuning, and the accuracy of the models on the different data sets.}
For the discovery of counterfactual examples, we consider observations $\mathbf{x}$ such that $f(\mathbf{x}) \neq t$, from the test sets of the cross-validation.

\subsection{Counterfactual search algorithms}
\rev{To provide experimental results concerning robustness (\Cref{sec:results-robustness}), we firstly seek a counterfactual search algorithm that performs best overall among several candidates.
To that end, we consider and benchmark the following algorithms from the literature, that can operate upon black-box $f$:}
\rev{\emph{Diverse Counterfactual Explanations} (DiCE)~\cite{mothilal2020dice},}
\emph{Growing Spheres} (GrSp)~\cite{laugel2018comparison}, \emph{LOcal Rule-based Explanations} (LORE)~\cite{guidotti2018local,guidotti2019factual}, 
and the \emph{Nelder-Mead method} (NeMe)~\cite{nelder1965simplex,gao2012implementing}. 
\rev{Furthermore, we devise our own algorithm, a genetic algorithm that we name \emph{Counterfactual Genetic Search} (CoGS)\footnote{\url{https://github.com/marcovirgolin/cogs}}.

The settings used for the algorithms are reported in \Cref{tab:alg-settings}.
We describe the algorithms below.
Note that all of the algorithms are heuristics with no guarantee of discovering optimal (i.e., minimal distance) counterfactual examples, given the nature of the search problem (general, black-box $f$).
}

\begin{table}[h]
    \centering
    \caption{
        Settings of the considered counterfactual search algorithms. For NeMe, we only set the maximum number of iterations to \num{100} to achieve commensurate runtimes to those of CoGS (other settings are default).
        \rev{For DiCE, we consider two configurations (``a'' and ``b'')}.
        The loss used (except for DiCE a) is \Cref{eq:loss-function}.
    }
    \label{tab:alg-settings}
    \begin{minipage}[t]{.4\linewidth}
        \centering
        \scalebox{0.90}{
        \begin{tabular}[t]{l c}
            \toprule
            \multicolumn{2}{c}{CoGS}\\
            Setting & {Value}\\
            \midrule
            Population size & \num{1000}\\
            Num.~generations & \num{100}\\
            Tournament size & \num{2}\\
            $s_\textit{mut}$ & $25\%$\\
            \bottomrule
        \end{tabular}
        }
    \end{minipage}\hfill%
    \begin{minipage}[t]{.6\linewidth}
        \centering
        \scalebox{0.9}{
        \begin{tabular}[t]{l c}
            \toprule
            \multicolumn{2}{c}{\rev{DiCE (a, b)}}\\
            Setting & {Value}\\
            \midrule
            Method & Genetic \\
            Total CEs & a : $20$, b : $100$\\
            Max.~iterations & a : $500$, b : $100$ \\
            \multirow{2}{*}{\makecell[l]{Loss weights}} & a : Default, \\
            & b : \small{$0.5$ prox., $0.5$ spars., $0$ div.}\\
            \bottomrule
        \end{tabular}
        }
    \end{minipage}\hfill%
    \begin{minipage}[t]{\linewidth}
    \mbox{}
    \end{minipage}
    \hfill%
    \begin{minipage}[t]{.4\linewidth}
    \centering
    \scalebox{0.90}{
    \begin{tabular}[t]{l c}
        \toprule
        \multicolumn{2}{c}{GrSp}\\
        Setting & {Value}\\
        \midrule
        Num.~in~layer & \num{2000}\\
        First radius & \num{0.1}\\
        Decrease radius & \num{10}\\
        Sparse & True\\
        \bottomrule
    \end{tabular}
    }
    \end{minipage}\hfill%
    \begin{minipage}[t]{.6\linewidth}
        \centering
        \scalebox{0.90}{
        \begin{tabular}[t]{l c}
            \toprule
            \multicolumn{2}{c}{LORE}\\
            Setting & {Value}\\
            \midrule
            Population size & \num{1000}\\
            Num.~generations & \num{10}\\
            Discrete use probabilities & False \\
            Continuous function estim. & False \\
            \bottomrule
        \end{tabular}
        }
    \end{minipage}\hfill%
        
\end{table}


\rev{
\subsubsection{DiCE}
DiCE is actually a library that includes three algorithms: random sampling, KD-tree search (i.e., a fast-retrieval data structure built upon the points in the training set), and a genetic algorithm. 
Of the three, we consider the latter because it performed substantially better in preliminary experiments (and simply refer to it by DiCE). 
DiCE is configured to return a collection of counterfactual examples rather than a single one.
However, three of the other algorithms we consider return a single counterfactual example.
Thus, to compare the algorithms on an equal footing, we set DiCE to return a single counterfactual example too.
We achieve this by ranking each counterfactual example in the collection according to the loss function in consideration (explained below, see \Cref{sec:loss}), and picking the best-ranking point.
We will further consider two different configuration of DiCE: 
\begin{itemize}
    \item \emph{Configuration ``a''} uses the default settings except for allowing for a longer number of iteration, to match the same computational budget given to the other algorithms. 
    \item \emph{Configuration ``b''} uses custom settings that are aligned to be similar to those used for CoGS, since both DiCE and CoGS are genetic algorithms.
\end{itemize}
}

\subsubsection{GrSp}
GrSp is a greedy algorithm that iteratively samples neighbors of the starting point $\mathbf{x}$ within spheres (i.e., in an $L2$ sense) that have increasing radius, until counterfactual examples are found.
GrSp includes feature selection to promote sparsity.
Unfortunately, GrSp can only handle numerical features.
To be able to use GrSp in our comparison, we let GrSp operate on categorical features as if they were numerical ones (categories are encoded as integers). 
At the end of the optimization, we transform numerical values back to categories by rounding.
\rev{Note that this is sub-optimal because an artificial ordering is introduced between categories.}

\subsubsection{LORE} 
LORE works by generating a neighborhood around $\mathbf{x}$ with random search or with a genetic algorithm, finding multiple counterfactual explanations at different distance.
We consider the variant that adopts the genetic algorithm, because it performed substantially better in preliminary experiments.
After the neighborhood is determined, LORE fits a decision tree upon it.
Since each path from the root of the decision tree to a leaf represents a classification rule (e.g., ``\texttt{AGE >= 3.4 \& SALARY\_CATEGORY = HIGH} $\rightarrow$ \texttt{$t$}''), LORE essentially returns multiple counterfactual explanations expressed as rules.
To be able to compare with the other algorithms (which return a single counterfactual example), we build one counterfactual example $\mathbf{z}$ by taking the shortest rule returned by LORE, and \emph{applying} the rule to the starting point $\mathbf{x}$ (e.g., using the rule above, we set the $\mathbf{x}$'s \emph{age} and \emph{salary} to $3.4$ and $\textit{high}$, respectively).

We found (confirmed by a discussion with the authors) that applying LORE's rules may results in points that are not actually classified as $t$.
When that happens, we perform up to 15 attempts at generating a counterfactual example from the (shortest returned) rule, by focusing on numerical features that are prescribed to be $>$, $\geq$ (or $<$, $\leq$) than a certain value.
In particular, in applying such part of the rule to $\mathbf{x}$, we add (or subtract) to the prescribed value a term $\epsilon$, which is initially set to $10^{-3}$ and is doubled at every attempt.
Moreover, since we found LORE to be computationally expensive to run (see \Cref{fig:runtimes-methods}), we used a fraction of the computation budget allowed for the other algorithms (see \Cref{tab:alg-settings}).

\subsubsection{NeMe}
\rev{NeMe is a classic simplex-based algorithm for gradient-free optimization.}
Like GrSp, also NeMe cannot naturally handle categorical features.
Thus, we use the same approximation used for GrSp, i.e., encode categories with integers, let NeMe treat categories as numerical values, and map such values back to integers (and thus categories) by rounding at the end.
We use SciPy's implementation with default parameters~\cite{virtanen2020scipy}.

\rev{
\subsubsection{CoGS}
We design CoGS as a relatively standard genetic algorithm,} adapted for the search of points neighboring $\mathbf{x}$ (especially in terms of the $L0$-norm).
\rev{CoGS operates as follows.}
First, an initial \emph{population} of candidate solutions is generated by sampling feature values uniformly within an interval for numerical features, and from the possible categories for categorical features.
\rev{These intervals can be specified or taken automatically from the training set.}
With probability of $2/d$ ($d$ being the total number of features), the feature value of a candidate solution is \emph{copied} from $\mathbf{x}$ rather than sampled.
Every iteration of the algorithm (in the jargon of evolutionary computation, \emph{generation}), offspring solutions are produced from the current population by \emph{crossover} and \emph{mutation}.
\rev{Following this, \emph{survival of the fittest} is applied to form the population for the next generation.}

Our version of crossover produces two offspring solutions by simply swapping the feature values of two random parents, uniformly at random.
Our version of mutation produces one offspring solution from one parent solution by randomly altering its feature values.
A feature value is altered with probability of $1/d$ (else, it is left untouched). 
If the feature to alter is categorical, then the category is swapped with another category, uniformly at random.
If the feature to alter is numerical, firstly a random number $r$ is sampled uniformly at random between $-s_\textit{mut}/2$ and $+s_\textit{mut}/2$, where $s_\textit{mut} \in (0,1]$ is a hyper-parameter that represents the maximal extent of allowed mutations; 
secondly, the original feature value is changed by adding $r \times (\max_i - \min_i)$, where $\max_i$ and $\min_i$ are, respectively, the maximum and minimum values that are possible for that feature. 

\rev{After crossover and mutation,} the quality (\emph{fitness}) of offspring solutions is evaluated using the loss function (\Cref{eq:loss-function}) as fitness function (minimization is sought).
\rev{Finally, we use} tournament selection~\cite{miller1995genetic}) \rev{to form the population for the next generation.}

\rev{We set CoGS to allow for plausibility constraints ($\mathcal{P}$) to be specified.}
If plausibility constraints are used, then mutation is restricted to plausible changes (e.g., the feature that represents age can only increase). 
If mutation makes a numerical feature obtain a value bigger than $\max_i$ (resp., smaller than $\min_i$), then the value of that feature is set to $\max_i$ (resp., $\min_i$).

CoGS is written in Python and relies heavily on NumPy~\cite{harris2020array} for the sake the speed (e.g., the population is encoded as a matrix and crossover and mutation operate upon it with matrix operations).

\subsection{Loss}
\label{sec:loss}
We use the following loss to drive the search of counterfactual examples (where $f(\mathbf{z})$ and $t$ are treated as integers):
\begin{align}
\label{eq:loss-function}
    \frac{1}{2} \gamma(\mathbf{z}, \mathbf{x}) &+ \frac{1}{2} \frac{|| \mathbf{z} - \mathbf{x} ||_0}{d} + || f(\mathbf{z}) - t ||_0, \text{where } \\
    \gamma(\mathbf{z},\mathbf{x}) &= \frac{1}{d} \left( \sum_{i}^{d_1} \frac{ | z_i - x_i | }{\max_i - \min_i} + \sum_{j}^{d_2} || z_j - x_j ||_0 \right).
\end{align}
The function $\gamma$ in the equation above is Gower's distance~\cite{gower1971general,d2021distances}, where features indexed by $i$ are numerical and those indexed by $j$ are categorical (with values treated as integers); 
the maximal and minimal values of a numerical feature, $\max_i$ and $\min_i$, can be taken from the (training) data set or, as done in our case, are provided as extra annotations of the data sets.
The term $|| \mathbf{z} - \mathbf{x} ||_0/d$ promotes sparsity of intervention and, like Gower's distance, ranges from zero to one.
The third and last term requires the execution of the machine learning model $f$, and simply returns zero when $f(\mathbf{z}) = t$ and one when $f(\mathbf{z}) \neq t$.

\rev{
\subsubsection{Incorporating robustness in the loss}
}
\label{sec:incorp-rob}
\rev{To seek robust counterfactual examples, we make use of the notions described in \Cref{sec:c-robust-cfe} and \Cref{sec:k-robust-cfe}.}
When optimizing for \rev{robustness to perturbations concerning} $\mathcal{C}$, \rev{we use \Cref{def:c-robust-cfe}, i.e.,} maximal $\mathcal{C}$-setbacks are computed on the fly for the candidate $\mathbf{z}$ and their contribution is used to update the contribution of $\gamma$ to the loss function. 
When optimizing for \rev{robustness to perturbations concerning} $\mathcal{K}$, we compute the $\mathcal{K}$-robustness score with \Cref{eq:practical-k-robust} and add $\frac{1}{2}\left( 1 - \mathcal{K}\text{-robustness score} \right)$ to the loss.
In the results presented below, we use $m=64$ to compute the $\mathcal{K}$-robustness score; 
an analysis on the impact of $m$ is provided in \ref{sec:apdx-sensitivity-k-robustness}.

\section{\rev{Preliminary Results: Choosing a Suitable Counterfactual Search Algorithm}}
\label{sec:results-algorithms}

\rev{This section reports on the benchmarking of the considered search algorithms, to identify an overall best.}
We repeat the execution of each algorithm five times and consider the best-found counterfactual example out of the five repetitions.
We search for a counterfactual example for each $\mathbf{x}$ in the test sets from the five cross-validation, for $\mathbf{x}$ such that $f(\mathbf{x}) \neq t$.
Since LORE takes much longer to execute than the other algorithms (see \Cref{fig:runtimes-methods}), we perform three repetitions instead of five, and consider only the first five $\mathbf{x}$ in each test set of the five folds.
Since only \rev{DiCE and} CoGS support plausibility constraints, we do not use plausibility constraints in this comparison.

\subsection{Runtimes}
\label{sec:apdx-runtime-algs}
\Cref{fig:runtimes-methods} shows the runtime of the algorithms across the different data sets, irrespective of whether they succeed or fail to find a counterfactual example, i.e., a point for which $f$ predicts $t$.
The experiments were run on a cluster where the computing nodes can have slightly different CPUs, thus we invite to consider the order of magnitude of the runtimes rather than the exact numbers.
The figure shows that\rev{, using random forest,} CoGS \rev{and DiCE (configuration a) are} the fastest algorithms (or, \rev{at least, have fastest implementations}), but GrSp and NeMe are competitive.
LORE is much slower to execute than the other algorithms. 
\rev{When using a neural network, inference times are generally faster, and CoGS, DiCE-a, GrSp and NeMe are competitive.}

\begin{figure}[h]
   \centering
   \begin{tabular}{lcc}
   \rotatebox{90}{\rev{\hspace{1cm}$f=\text{random forest}$}} &
   \includegraphics[width=0.45\linewidth]{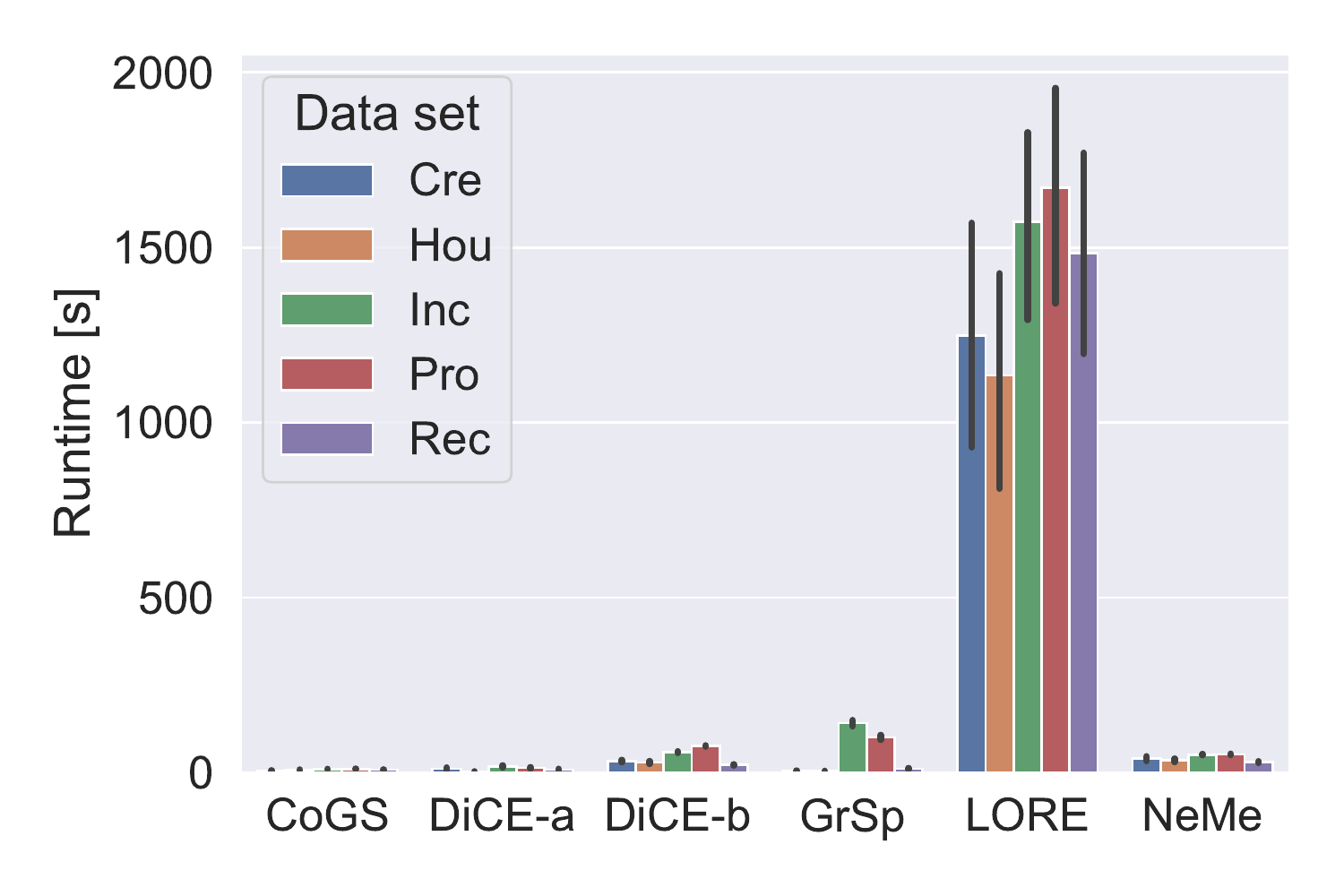}  &  \includegraphics[width=0.45\linewidth]{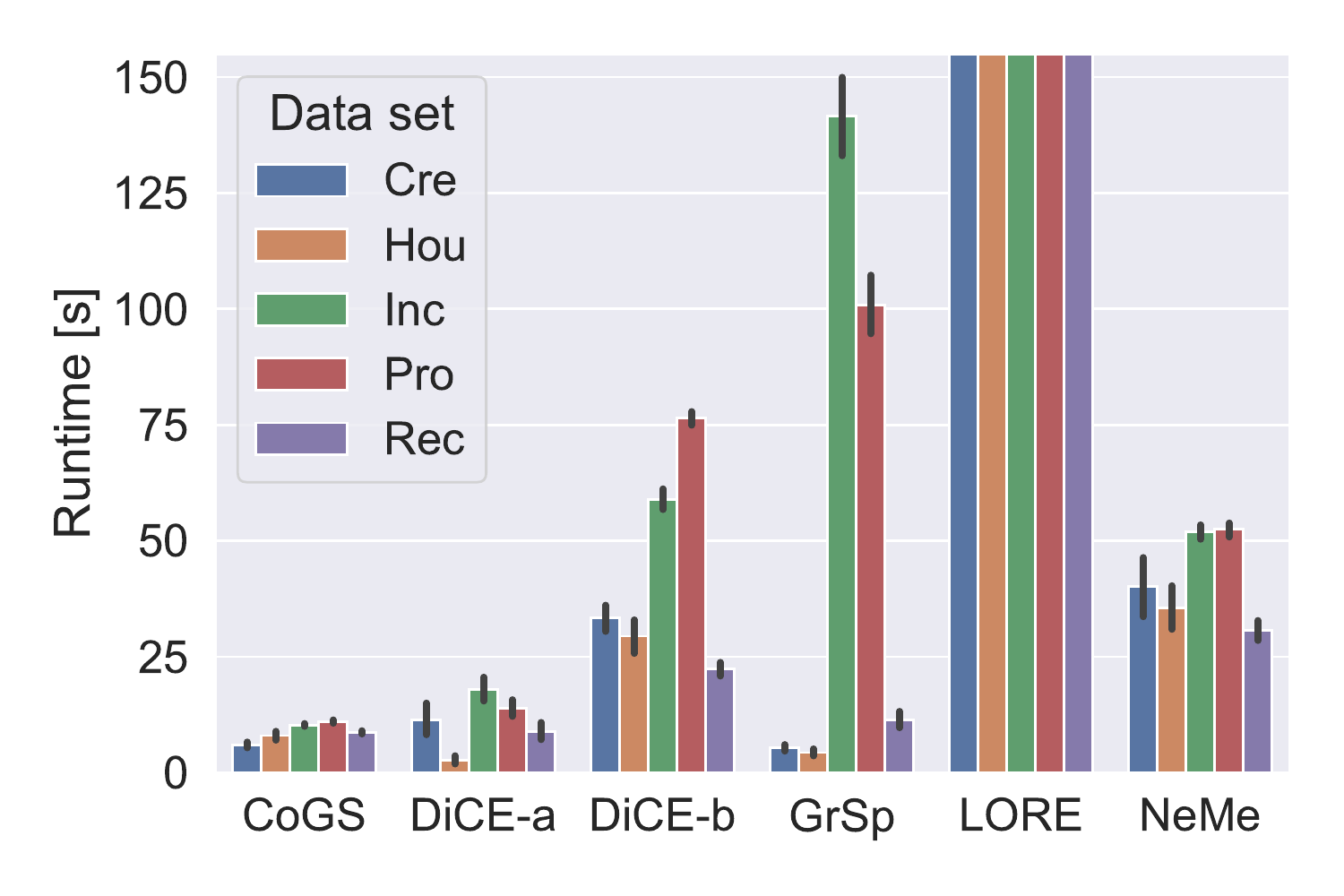}\\
   \rotatebox{90}{\rev{\hspace{1cm}$f=\text{neural network}$}} &
   \includegraphics[width=0.45\linewidth]{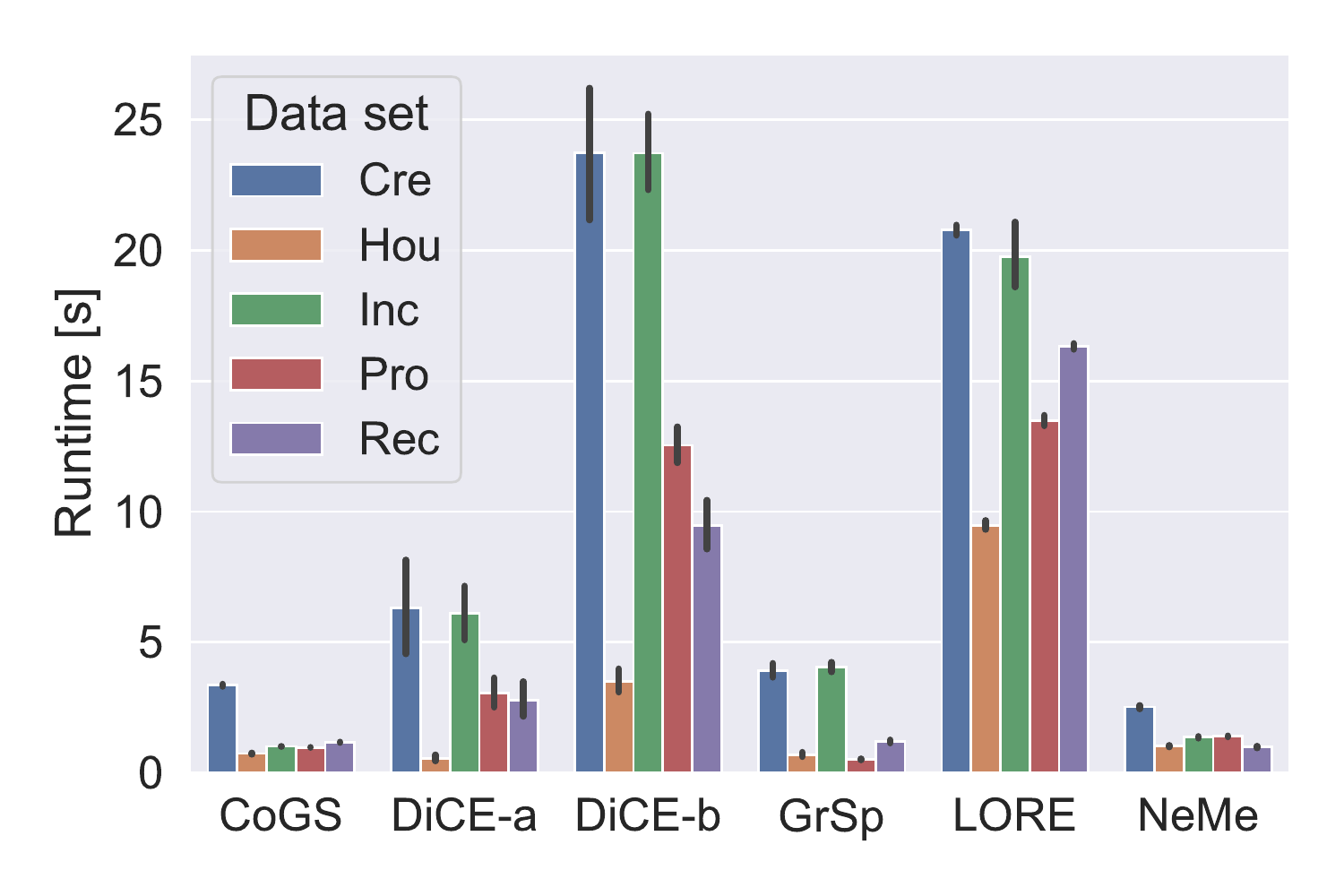}  &  \includegraphics[width=0.45\linewidth]{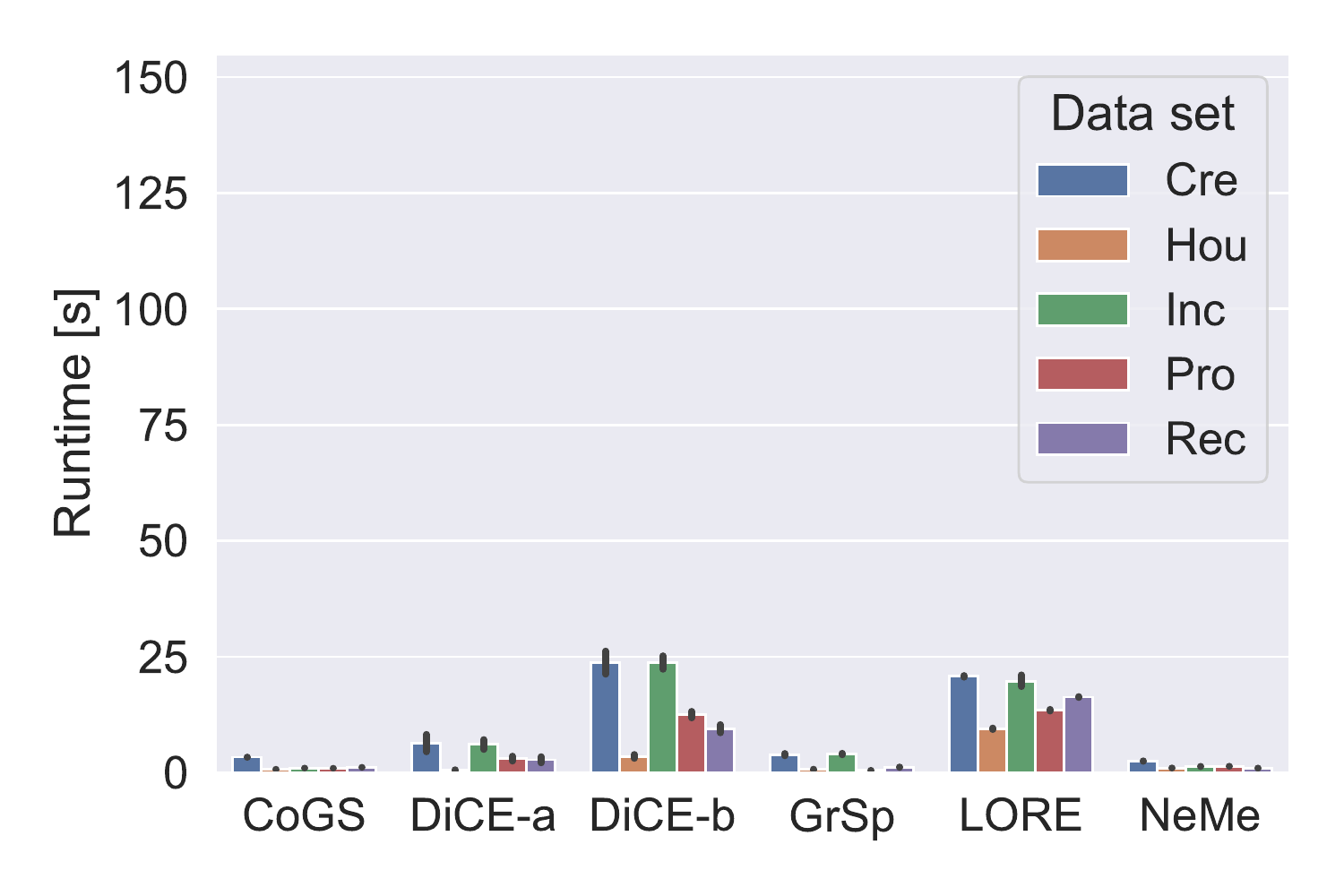}
   \end{tabular}
   \caption{Runtimes (means and 95\% confidence intervals) of the counterfactual search algorithms for the considered data sets \rev{and black-boxes (random forest and neural network)}. 
   The right plots are zoomed-in versions of the left ones.}
    \label{fig:runtimes-methods}
\end{figure}

\subsection{Success in discovering counterfactual examples}
\Cref{tab:success-rates} shows the frequency with which the counterfactual search algorithms succeed in finding a counterfactual example, i.e., a point for which $f$ predicts $t$.
CoGS \rev{and the two variants of DiCE} succeed systematically, whereas the other algorithms do not.
GrSp performs \rev{third}-best overall. 
In particular, GrSp always finds counterfactual examples on Hou, which is a data set with a single categorical feature. 
Since GrSp is intended to operate solely with numerical features, this results nicely supports the hypothesis that GrSp works well when (almost all) features are numerical.
Although LORE supports both numerical and categorical features, it does not perform better than GrSp on most data sets; 
at least for the limited number of runs conducted with LORE due to excessive runtime, as explained before. 
Lastly, NeMe often performs substantially worse than all other algorithms.

\begin{table}[h]
    \centering
    \caption{Mean $\pm$ standard deviation across five cross-validation folds of the frequency with which the counterfactual search algorithms succeed in finding a counterfactual example. 
    Plausibility constraints are not considered here because not all algorithms support them.}
    \sisetup{separate-uncertainty,scientific-notation=fixed,round-mode=places}
    
    \begin{tabular}{llccccc}
    \toprule
    & {Alg.} & {Cre} & {Inc} & {Hou} & {Pro} & {Rec} \\
    \midrule
    \multirow{6}{*}{\rotatebox{90}{\scriptsize $f=\text{random forest}$}}
     & CoGS & 1.00 (0) & 1.00 (0) & 1.00 (0) & 1.00 (0) & 1.00 (0)  \\ 
     & DiCE-a & 1.00 (0) & 1.00 (0) & 1.00 (0) & 1.00 (0) & 1.00 (0)  \\ 
     & DiCE-b & 1.00 (0) & 1.00 (0) & 1.00 (0) & 1.00 (0) & 1.00 (0)  \\ 
     & GrSp & 0.46 (11) & 0.89 (6) & 1.00 (0) & 0.86 (4) & 0.30 (15)  \\ 
     & LORE & 0.56 (20) & 0.20 (13) & 0.68 (20) & 0.24 (20) & 0.60 (38)  \\ 
     & NeMe & 0.08 (3) & 0.05 (2) & 0.04 (5) & 0.03 (1) & 0.14 (2)  \\ 
    \midrule
    \multirow{6}{*}{\rev{\rotatebox{90}{\scriptsize $f=\text{neural network}$}}}
     & CoGS & 1.00 (0) & 1.00 (0) & 1.00 (0) & 1.00 (0) & 1.00 (0)  \\ 
     & DiCE-a & 1.00 (0) & 1.00 (0) & 1.00 (0) & 1.00 (0) & 1.00 (0)  \\ 
     & DiCE-b & 1.00 (0) & 1.00 (0) & 1.00 (0) & 1.00 (0) & 1.00 (0)  \\ 
     & GrSp & 0.87 (7) & 0.25 (4) & 1.00 (0) & 0.51 (12) & 0.49 (10)  \\ 
     & LORE & 0.52 (20) & 0.28 (16) & 0.68 (10) & 0.76 (23) & 0.84 (23)  \\ 
     & NeMe & 0.14 (7) & 0.11 (3) & 0.09 (1) & 0.11 (4) & 0.51 (3)  \\ 
    \bottomrule
    \end{tabular}
   
    \label{tab:success-rates}
\end{table}

\subsection{Quality of discovered counterfactual examples}
\rev{As last part in our benchmarking effort, we consider what algorithm manages to produce near-optimal counterfactual examples (i.e., those with smallest loss).}
\rev{In particular, we report} the relative change in loss for the best-found counterfactual example with respect to the loss obtained by CoGS, only for success cases.
\rev{Since we consider only successes, the last term of the loss (\Cref{eq:loss-function}) is always null, i.e.,  $|| f(\mathbf{z}) - t ||_0 = 0$.}
The relative change in loss with respect to CoGS for another algorithm \emph{Alg} is:
\begin{equation*}
    \frac{\mathcal{L}_\textit{Alg}(\mathbf{z}) - \mathcal{L}_\text{CoGS}(\mathbf{z}) }{\mathcal{L}_\text{CoGS}(\mathbf{z})}.
\end{equation*}

\Cref{fig:rel-change-loss} shows the relative change in loss of \rev{DiCE,} GrSp, LORE, and NeMe with respect to CoGS.
\rev{DiCE,} GrSp and LORE typically (but not always) find points that have larger loss than those found by CoGS. 
NeMe performs very similarly to CoGS, however NeMe seldom succeeds (cfr.~\Cref{tab:success-rates}). 
This suggests that NeMe can explore a small neighborhood of $\mathbf{x}$ particularly well, but fails if counterfactual examples are relatively distant from $\mathbf{x}$.

\begin{figure}[h]
    \centering
    \begin{tabular}{lc}
    \rotatebox{90}{\hspace{.75cm}$f=\text{random forest}$}
    &
    \includegraphics[width=.5\linewidth]{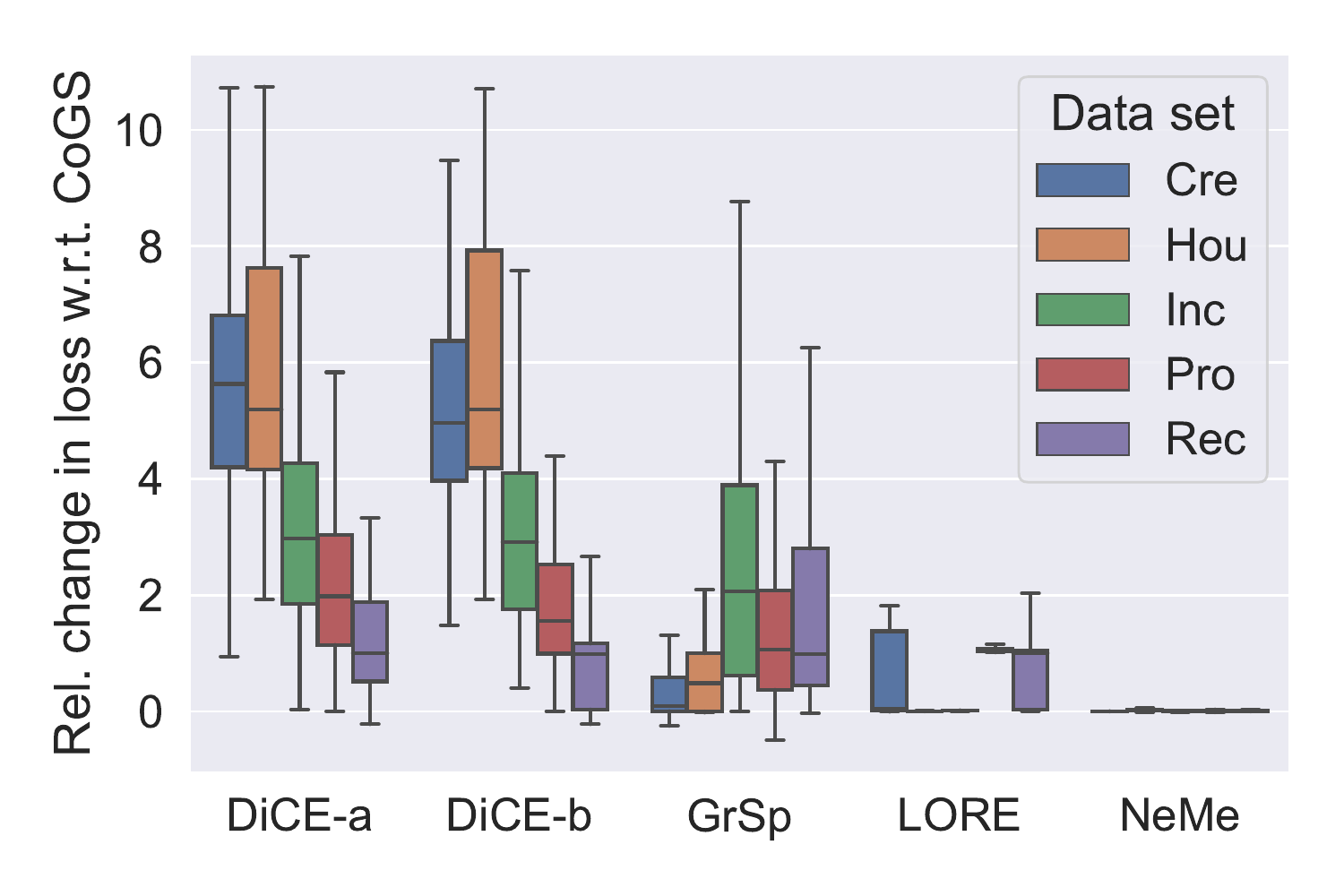} \\
    \rotatebox{90}{\hspace{.5cm}$f=\text{neural network}$}
    &
    \includegraphics[width=.5\linewidth]{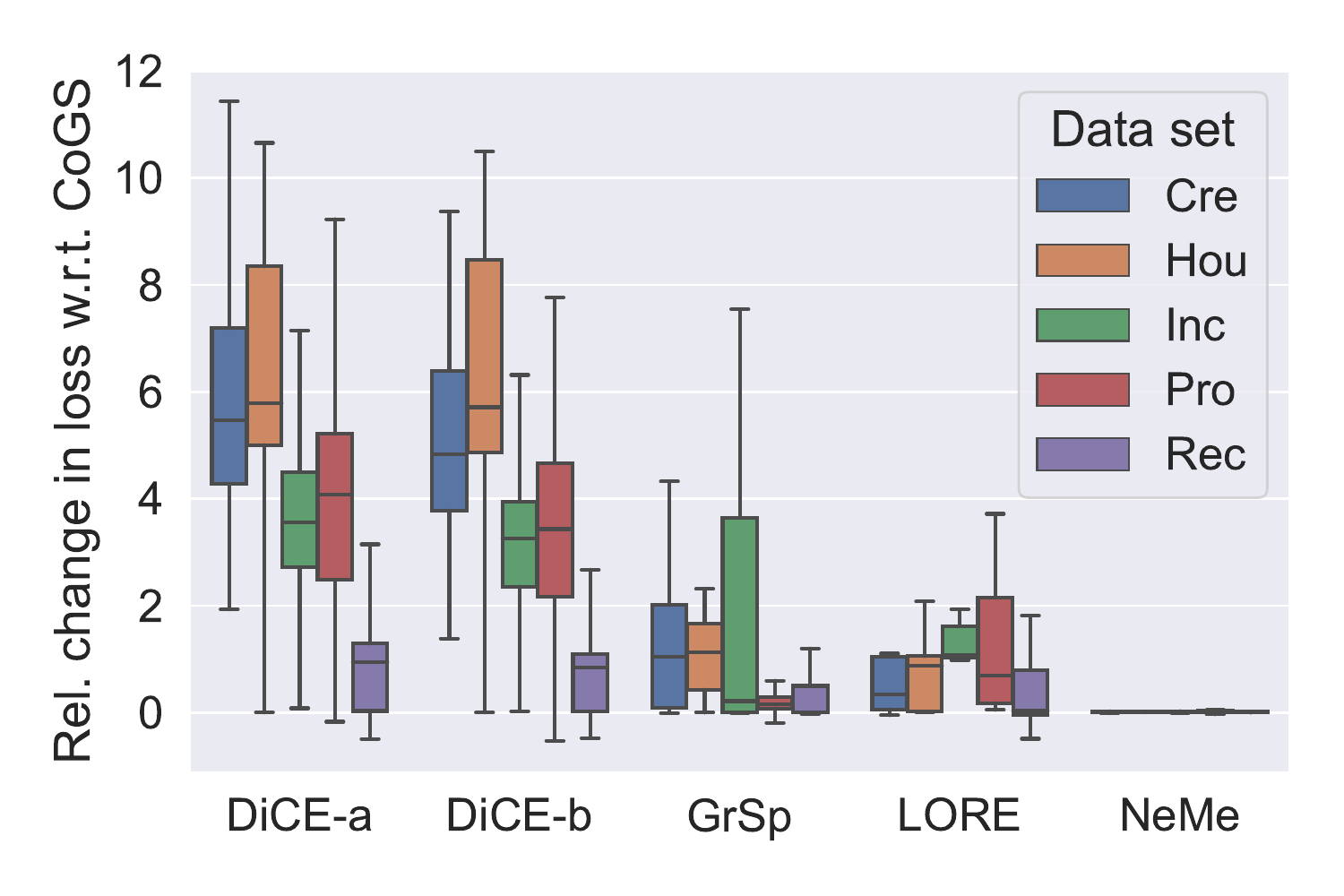} \\
    \end{tabular}
    \caption{Boxplots of relative change in loss with respect to CoGS for GrSp, LORE, and NeMe, on the different data sets \rev{and black-boxes} for success cases.}
    \label{fig:rel-change-loss}
\end{figure}

\rev{
\subsection{Conclusion of benchmarking}
The results show that, overall, CoGS performs best.
DiCE (in particular, DiCE-a) is the closest competitor in terms of speed and success rates, but the algorithm finds counterfactual examples  that are substantially more distant from $\mathbf{x}$ (i.e., have larger loss) than those found by CoGS.
GrSp has good runtime and generally finds closer counterfactual examples (i.e., lower loss) than DiCE, but it remains inferior to CoGS, both in terms of distance (loss) and success rate.
LORE has worse success rate than GrSp, and NeME worse of all.
Therefore, we use CoGS for the following experiments on robustness.
}

\section{Experimental Results: Robustness}
\label{sec:results-robustness}

\rev{We proceed with presenting the experimental results regarding robustness to perturbations in $\mathcal{C}$,  $\mathcal{K}$, and jointly}.
\rev{We focus on results that allow us to answer} what we believe to be important research questions: 
(RQ1) \emph{Do we need to account for robustness to discover robust counterfactual examples?}
(RQ2) \emph{Does a lack of robustness compromise the feasibility of correcting perturbations with additional intervention?}
(RQ3) \emph{Are robust counterfactual explanations advantageous in terms of additional intervention cost?}
These questions are addressed, in order, in the next subsections.
Because of space limitations, a number of additional results is reported in \ref{sec:apdx-additional-results}\rev{, including runtime taken to account for robustness w.r.t.~$\mathcal{C}$ and $\mathcal{K}$, and the effect of varying $m$ when computing the $\mathcal{K}$-robustness score}.
We \rev{now account for plausibility constraints} $\mathcal{P}$ in \rev{all of the} following experiments.
We remark that in all our experiments, CoGS always succeeded in discovering a counterfactual example for which $f$ predicts $t$\rev{, except for having a mean success rate of $99\%$ (st.dev.~of $1\%$) on the Rec data set when $f$ is implemented as a neural network}.

\subsection{(RQ1) Do we need to account for robustness to discover robust counterfactual examples?}
\label{sec:results-rq1}

\Cref{tab:cogs-happen-to-find-robust} shows the frequency with which robust counterfactual examples are discovered accidentally.
To realize this, we compare the best-found counterfactual example that is discovered by CoGS when robustness \emph{is not} accounted for, and the one that is found when $\mathcal{C}$- or $\mathcal{K}$-robustness \emph{is} accounted for \rev{(as indicated in \Cref{sec:incorp-rob})}.
\rev{We take the frequency by which the two match as indication of whether robust counterfactual examples can be discovered by accident.}
Since numerical feature values may differ only slightly between two best-found counterfactual examples, we consider the values to match if they are sufficiently close to each other, according to a tolerance level of \rev{$1\%$, $5\%$, or $10\%$} of the range of that feature. 
\rev{As reasonable to expect, the results show that the larger the tolerance level, the more a $\mathbf{z}^\star$ discovered when not accounting for robustness matches the respective one that is discovered when accounting for robustness. 
In general, the result depends on the data set in consideration, and also (albeit arguably less so) on whether random forest or a neural network is used as black-box model $f$.

For brevity, we now focus on the tolerance level of $5\%$ and random forest.}
On Inc, best-found counterfactual examples \rev{rarely match with those discovered when accounting for  $\mathcal{C}$-robustness} ($4\%$ \rev{on average for the tolerance of $5\%$})\rev{, while the vice versa happens on Hou ($84\%$ on average for the same tolerance).}
\rev{For $\mathcal{K}$-robustness like} for $\mathcal{C}$-robustness, the result depends on the data set.
Importantly, the data sets where the frequencies are high for $\mathcal{C}$-robustness and $\mathcal{K}$-robustness are not necessarily the same.
On Inc, best-found counterfactual examples are rarely optimal under $\mathcal{C}$-setbacks, but \rev{can often match} with counterfactual examples discovered when penalizing low $\mathcal{K}$-robustness scores (\rev{$40\%$ on average for the tolerance of $5\%$}).
This should not be surprising because $\mathcal{C}$- and $\mathcal{K}$-robustness are orthogonal to each other under the assumption of feature independence.
The last row shows how often best-found counterfactual examples happen to be both robust to perturbations to $\mathcal{C}$ and $\mathcal{K}$. 
The frequencies are clearly always lower than for the previous triplets of rows.
Hou is the only data set for which the frequency of discovering a counterfactual example that happens to be both robust w.r.t.~$\mathcal{C}$ and $\mathcal{K}$ by chance is relatively large (\rev{e.g.,} above $50\%$ \rev{for the tolerance of $5\%$}).

\rev{
When using the neural network instead of random forest, the trends mentioned before remain the same, but the specific magnitudes can differ.
For example, the accidental discovery of robust counterfactual examples w.r.t.~$\mathcal{C}$ and/or $\mathcal{K}$ is lower on Cre with the neural network compared to random forest, but the opposite holds for Hou (with some exceptions, e.g., the tolerance level of $1\%$ when both $\mathcal{C}$- and $\mathcal{K}$-robustness are sought).

Overall, this result indicates that, except for lucky cases (e.g., Hou with $f$ being the neural network), it is unlikely to discover robust counterfactual examples by chance.
Hence, if one wishes to achieve robustness, the search must be explicitly instructed to that end.
In the next sections, we investigate whether achieving robustness can actually be important.
}


\begin{table}[]
    \centering
    \caption{Mean $\pm$ standard deviation of the frequency with which the best-found (among five search repetitions) counterfactual example when not accounting for robustness is accidentally robust w.r.t.~$\mathcal{C}$ or $\mathcal{K}$.
    For numerical features, we consider them to match in value if they are within a tolerance level (Tol.) of 1\%, 5\% or 10\% of the range for that feature.}
    \sisetup{separate-uncertainty,scientific-notation=fixed,round-mode=places}
    \resizebox{\linewidth}{!}{%
    \begin{tabular}{llcccccc}
\toprule
& {Robustness} & {Tol.} &  {Cre} & {Inc} & {Hou} & {Pro} & {Rec} \\
\midrule
\multirow{9}{*}{\rotatebox{90}{$f=\text{random forest}$}}
& \multirow{3}{*}{Only $\mathcal{C}$}
 & 1\% & 0.40 (6) & 0.02 (0) & 0.76 (10) & 0.53 (5) & 0.27 (6) \\
& & 5\% & 0.42 (7) & 0.04 (2) & 0.84 (9) & 0.57 (6) & 0.37 (7) \\
& & 10\% & 0.43 (7) & 0.05 (2) & 0.85 (9) & 0.58 (6) & 0.40 (9) \\
\cline{2-8}
& \multirow{3}{*}{Only $\mathcal{K}$}
 & 1\% & 0.37 (1) & 0.06 (2) & 0.33 (24) & 0.26 (5) & 0.04 (4) \\
& & 5\% & 0.44 (3) & 0.40 (8) & 0.63 (17) & 0.37 (6) & 0.08 (3) \\
& & 10\% & 0.46 (4) & 0.58 (7) & 0.67 (16) & 0.46 (7) & 0.12 (2) \\
\cline{2-8}
& \multirow{3}{*}{Both $\mathcal{C},\mathcal{K}$}
 & 1\% & 0.23 (4) & 0.00 (0) & 0.21 (21) & 0.19 (6) & 0.03 (3) \\
& & 5\% & 0.27 (3) & 0.00 (0) & 0.54 (21) & 0.26 (5) & 0.06 (4) \\
& & 10\% & 0.30 (5) & 0.00 (0) & 0.60 (19) & 0.34 (6) & 0.08 (4) \\
\midrule
\multirow{9}{*}{\rev{\rotatebox{90}{$f=\text{neural network}$}}}
& \multirow{3}{*}{Only~$\mathcal{C}$}
 & 1\% & 0.25 (12) & 0.01 (1) & 0.96 (2) & 0.87 (5) & 0.50 (8) \\
& & 5\% & 0.27 (12) & 0.02 (1) & 0.97 (2) & 0.89 (5) & 0.56 (5) \\
& & 10\% & 0.29 (11) & 0.02 (1) & 0.97 (2) & 0.89 (5) & 0.57 (4) \\
\cline{2-8}
& \multirow{3}{*}{Only~$\mathcal{K}$}
 & 1\% & 0.13 (7) & 0.35 (2) & 0.07 (5) & 0.08 (6) & 0.01 (0) \\
& & 5\% & 0.26 (8) & 0.52 (3) & 0.80 (12) & 0.42 (19) & 0.01 (1) \\
& & 10\% & 0.39 (2) & 0.70 (4) & 0.93 (4) & 0.58 (14) & 0.02 (2) \\
\cline{2-8}
& \multirow{3}{*}{Both~$\mathcal{C},\mathcal{K}$}
 & 1\% & 0.02 (2) & 0.00 (0) & 0.07 (5) & 0.06 (6) & 0.00 (1) \\
& & 5\% & 0.06 (3) & 0.00 (0) & 0.69 (9) & 0.38 (18) & 0.01 (1) \\
& & 10\% & 0.12 (8) & 0.00 (0) & 0.93 (4) & 0.52 (14) & 0.01 (2) \\
\bottomrule
    \end{tabular}
    }
    \label{tab:cogs-happen-to-find-robust}
\end{table}

\subsection{(RQ2) Does a lack of robustness compromise the feasibility of correcting perturbations with additional intervention?}
\label{sec:results-rq2}

\begin{figure}
    \centering
    \begin{tabular}{cc}
        Uniformly-distributed perturbations, \rev{$f=\text{random forest}$} \\ 
        \includegraphics[width=0.75\linewidth]{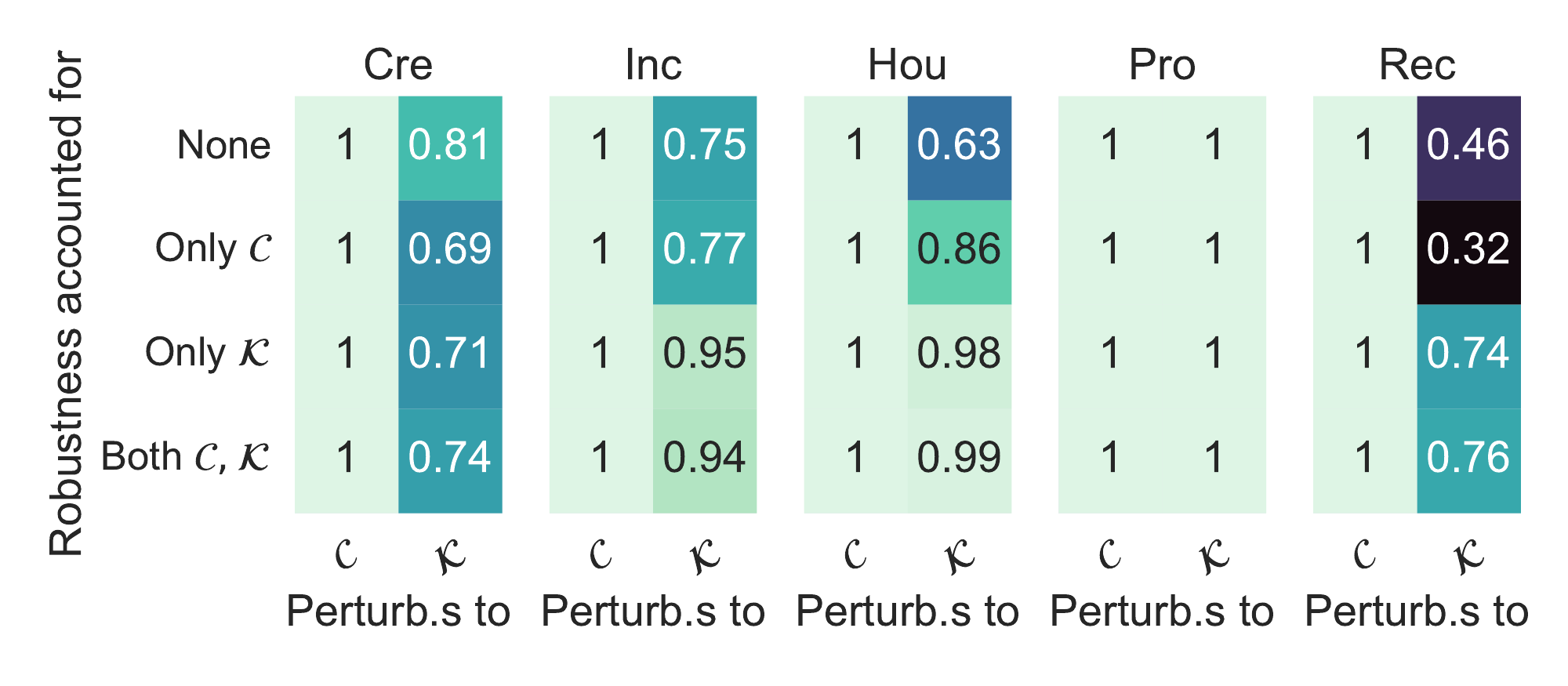} \\
        Normally-distributed perturbations, \rev{$f=\text{random forest}$}\\ \includegraphics[width=0.75\linewidth]{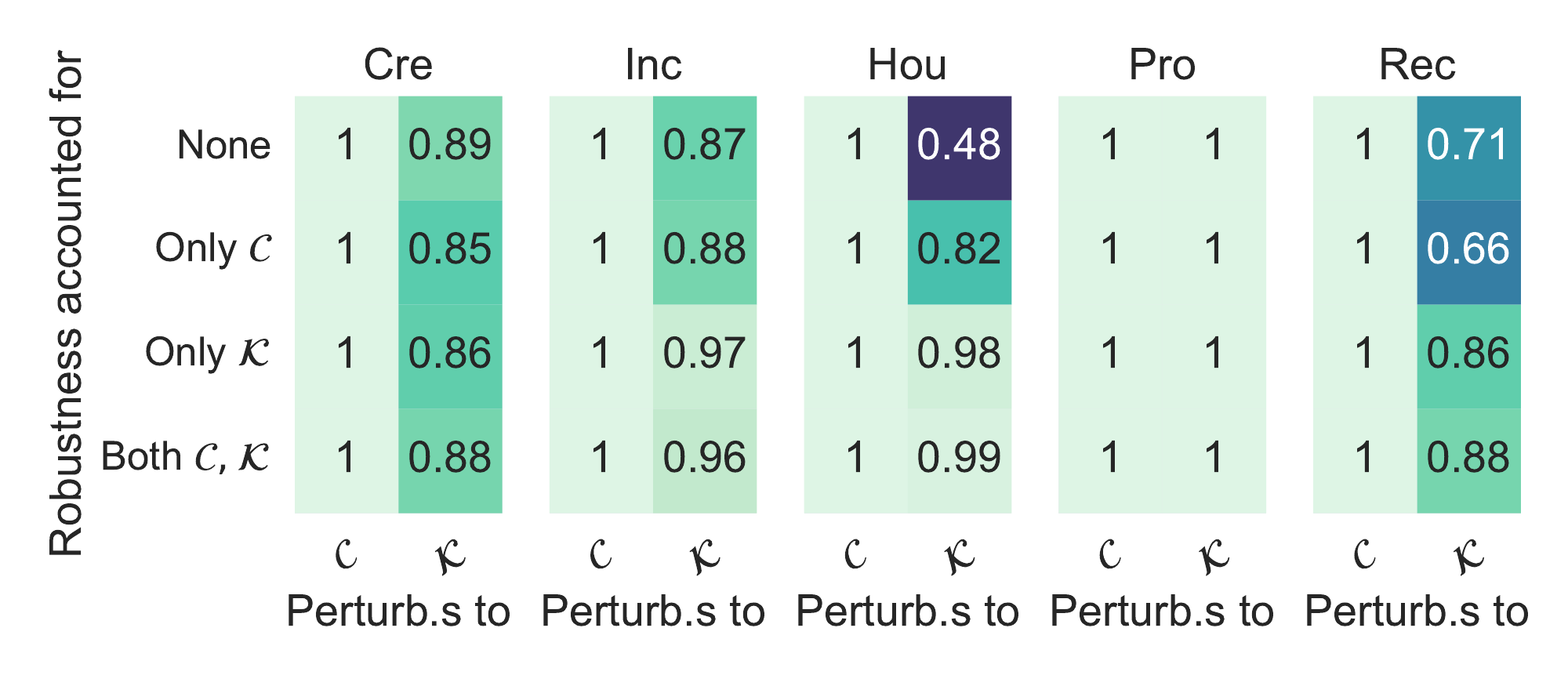}\\
    \end{tabular}
    \caption{Mean frequency with which a plausible additional intervention exists, to contrast the perturbations and reach the intended counterfactual example (uniformly-distributed categorical changes and normally- or uniformly-distributed numerical changes, \rev{$f=\text{random forest}$}).
    Darker colors represent worse cases.}
    \label{fig:fixability-rf}
\end{figure}

\begin{figure}
    \centering
    \begin{tabular}{cc}
        Uniformly-distributed perturbations, \rev{$f=\text{neural network}$} \\ 
        \includegraphics[width=0.75\linewidth]{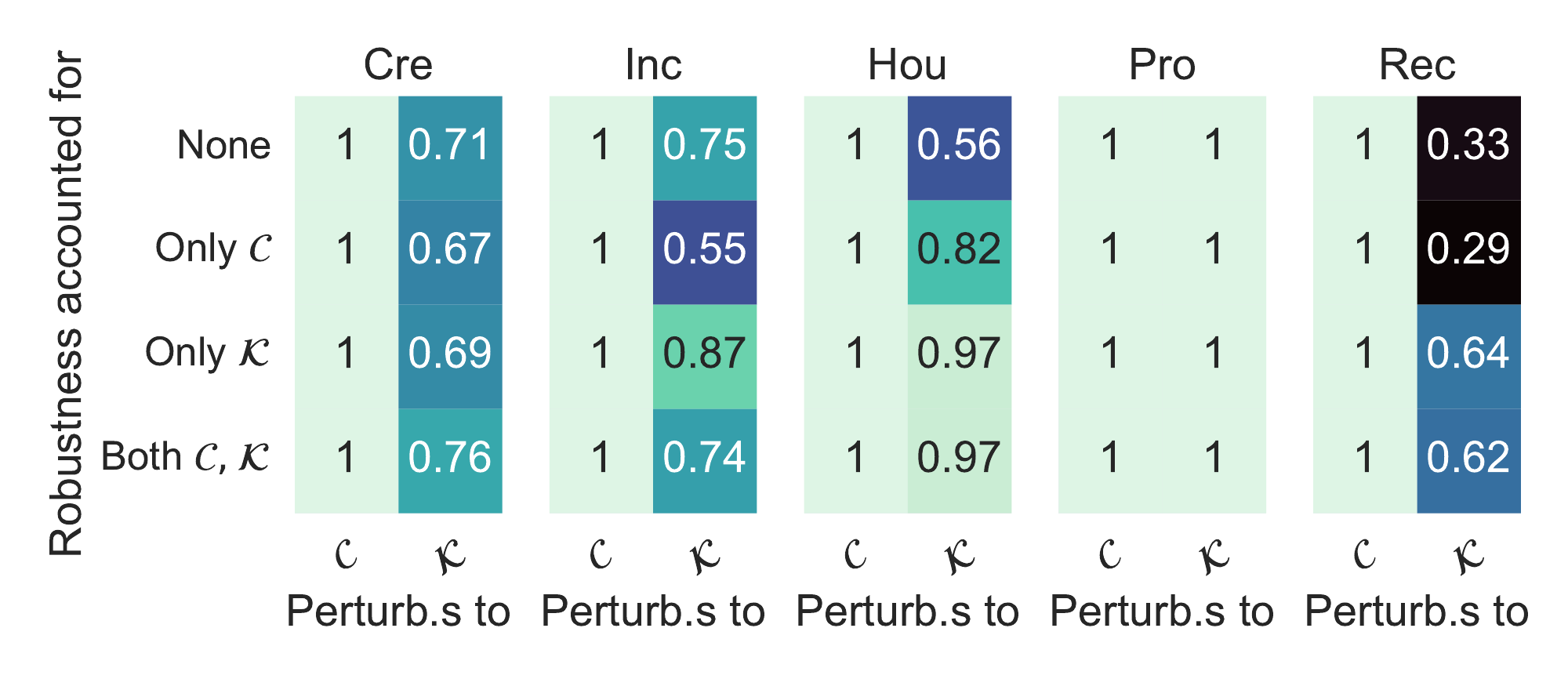} \\
        Normally-distributed perturbations, \rev{$f=\text{neural network}$}\\ \includegraphics[width=0.75\linewidth]{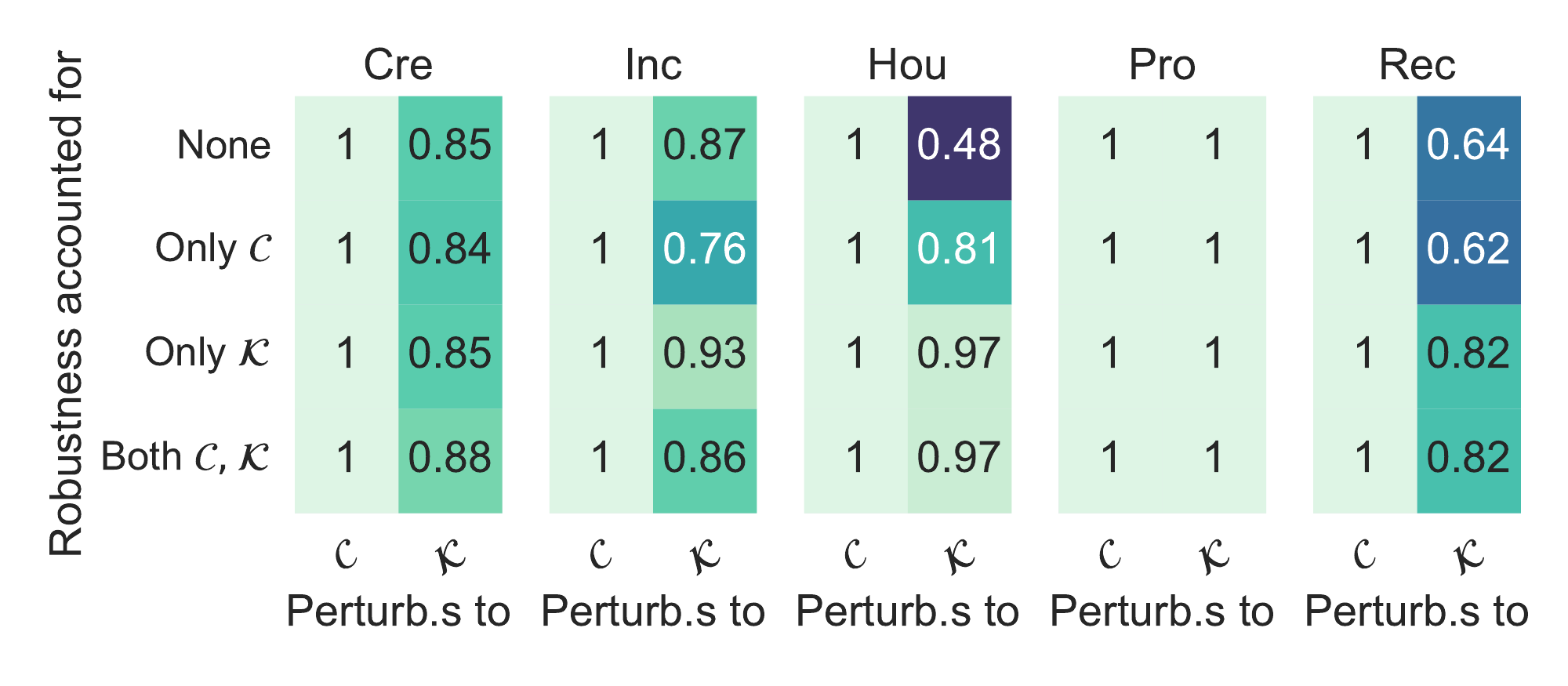}\\
    \end{tabular}
    \caption{Mean frequency with which a plausible additional intervention exists, to contrast the perturbations and reach the intended counterfactual example (uniformly-distributed categorical changes and normally- or uniformly-distributed numerical changes, \rev{$f=\text{neural network}$}).
    Darker colors represent worse cases.}
    \label{fig:fixability-nn}
\end{figure}

At this point, current works on the robustness of counterfactual explanations typically consider the extent by which robustness helps preventing the invalidation of counterfactual explanations (see \Cref{sec:related}).
In other words, they consider whether the point $\mathbf{z}^\prime$ that is given by perturbing the best-found counterfactual example is still classified as $t$.
For completeness, we report on this in \ref{sec:apdx-invalidity}.
Current works do not, however, consider whether an additional intervention that allows to correct the perturbation and obtain $t$ might exist.


\Cref{fig:fixability-rf,fig:fixability-nn} show the frequency with which achieving the intended counterfactual explanation remains possible after random perturbations take place.
The frequency is computed by applying, for each counterfactual explanation outcome of the search, \num{100} perturbations that are sampled uniformly at random from the categorical possibilities for categorical features, and normally (with st.dev.~of \num{0.1}) or uniformly within the numerical intervals for numerical features, as defined in $\mathbf{p}$.
\rev{We note that similar results are obtained between choosing random forest or a neural network as $f$.}

As expected, it is always possible to contrast $\mathcal{C}$-setbacks, \rev{because these happen along the direction of intervention}.
Instead, perturbations concerning $\mathcal{K}$ can lead to a $\mathbf{z}^\prime$ such that \rev{no further plausible intervention exists to reach the originally intended counterfactual example}. 
We do not report a result for perturbations concerning both $\mathcal{C}$ and $\mathcal{K}$ \rev{at the same time} because, by construction, it is the same as the result for perturbations concerning only $\mathcal{K}$.
Like for the results of \Cref{sec:results-rq1}, the extent by which perturbations to $\mathcal{K}$ reduce the possibility for further intervention depends on the data set. 
On Pro, all perturbations can be contrasted by an additional intervention because there are no plausibility constraints (see \Cref{tab:datasets}).
Conversely, on Rec, perturbations to $\mathcal{K}$ can often make it impossible to reach the originally-intended counterfactual example, unless $\mathcal{K}$-robustness is accounted for.
In fact, accounting for $\mathcal{K}$-robustness generally improves the chances that further intervention is possible, at times substantially (e.g., on Inc, Hou, and Rec).
Cre represents the only exception to this, as accounting for $\mathcal{K}$-robustness \rev{performs similar (or sometimes worse) than accounting for none}. 
This suggests that the decision boundary learned by $f$ on this data set may not be very smooth, making the use of the $\mathcal{K}$-robustness score \rev{a too coarse approximation to be helpful}.
\rev{Generally, accounting for perturbations to $\mathcal{C}$ \emph{alone} does not help achieving substantial robustness to perturbations to $\mathcal{K}$, except for on Hou.
This suggests that, on Hou, $f$ learns decision boundaries that incorporate interesting interactions between certain features.
}
Importantly, accounting for $\mathcal{C}$-robustness \emph{together} with accounting for $\mathcal{K}$-robustness does not \rev{substantially} compromise the gains obtained by accounting for $\mathcal{K}$-robustness alone, even though perturbations to $\mathcal{C}$ always admit further intervention.
Overall, these results show that accounting for robustness can be crucial to ensure that, if perturbations happen, additional intervention to obtain $t$ remains possible.

\subsection{(RQ3) Are robust counterfactual explanations advantageous in terms of additional intervention cost?}
\label{sec:results-rq3}

We present the following results in terms of a \emph{relative cost}, namely, the ratio between the cost of intervention to reach \rev{the intended} $\mathbf{z}$ when random perturbations take place (i.e., initial the cost of reaching $\mathbf{z}$ from $\mathbf{x}$ plus the cost of reaching $\mathbf{z}$ from the perturbed $\mathbf{z}^\prime$), and the \rev{\emph{ideal cost}}, i.e., the cost \rev{incurred in complete absence of perturbations} (i.e., the cost of reaching $\mathbf{z}$ from $\mathbf{x}$).
We compute this relative cost when the notions of robustness are or are not accounted for.
The ideal cost is computed when \emph{not} accounting robustness.
The cost is modeled by $\frac{1}{2}\gamma(\mathbf{z},\mathbf{x}) + \frac{1}{2}\frac{||\mathbf{z} - \mathbf{x}||_0}{d}$ (i.e., the first part of \Cref{eq:loss-function}).
Moreover, if $f(\mathbf{z}^\prime)=t$, we assume no additional intervention to be needed, and thus \rev{the additional cost is zero and the relative cost is $1$}.

\begin{figure}
    \centering
    \setlength{\tabcolsep}{0pt}
    \begin{tabular}{lc}
    & Uniformly-distributed perturbations\\
    \multirow{2}{*}{\rotatebox{90}{\rev{$f=\text{random forest}$}}} & 
    \includegraphics[width=0.98\linewidth]{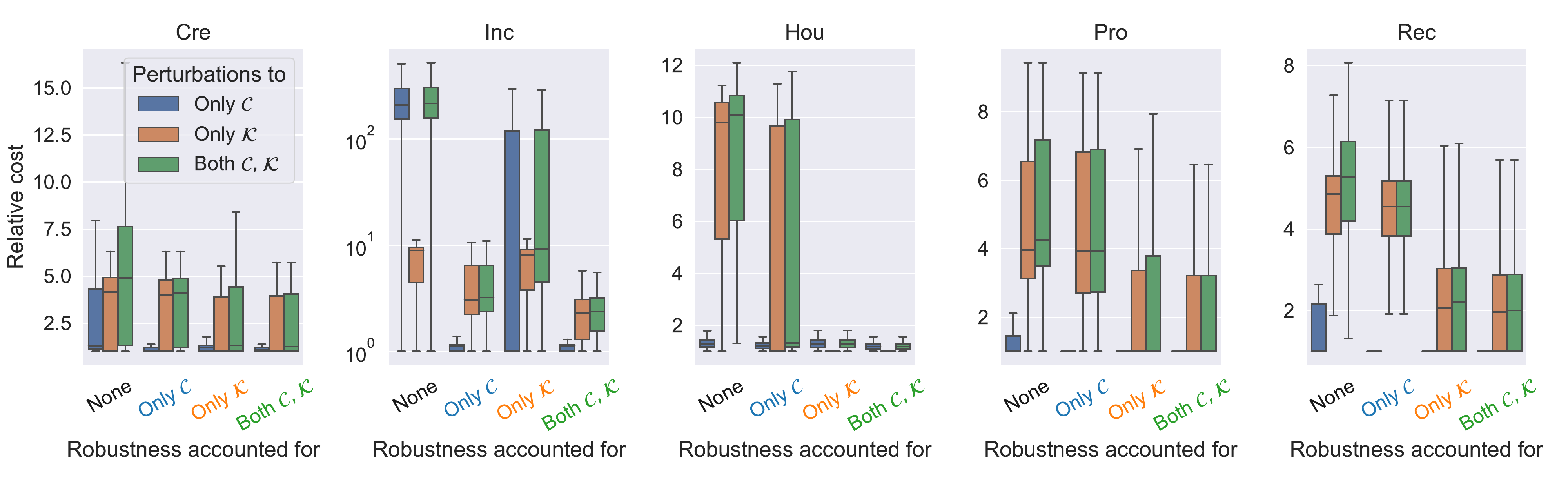}\\
    & Normally-distributed perturbations\\
    & \includegraphics[width=0.98\linewidth]{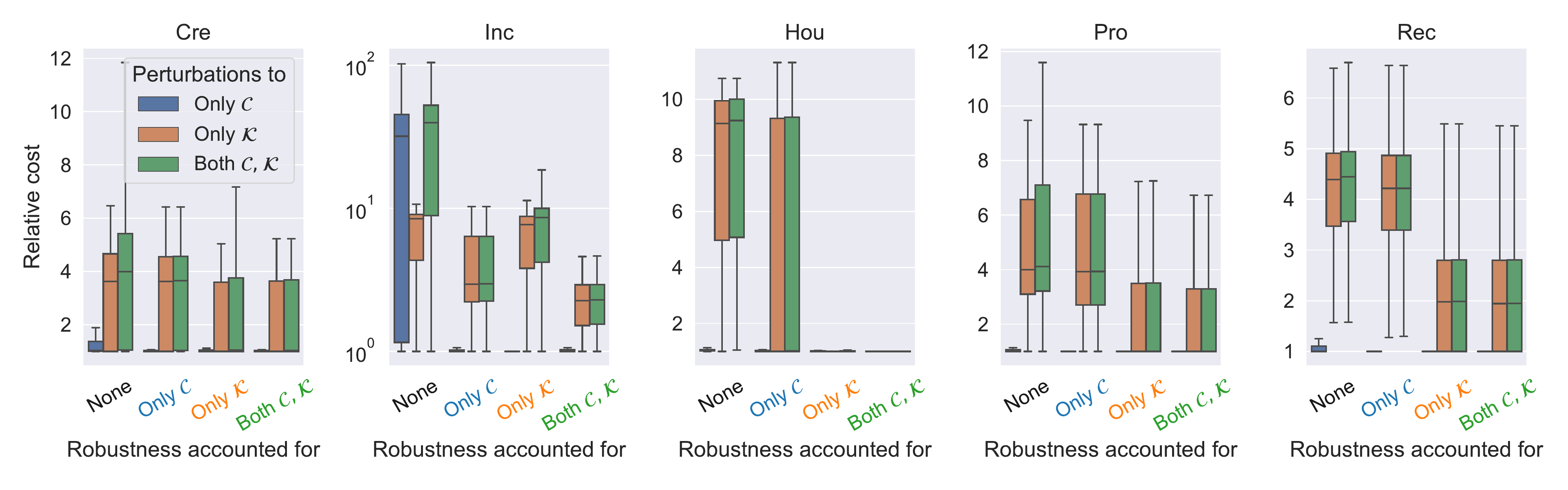}\\
    \end{tabular}
    \caption{
    Cost in terms of different configurations of accounting for robustness and under different perturbations, relative to the ideal cost \rev{(with random forest)}.
    Due to perturbations, t perturbations, the relative cost for when no notion of robustness is accounted for (label None) is typically much larger than the one for when the right notion of robustness is accounted for (matching color between box and label).
    The vertical axis for Inc is in logarithmic scale.}
    \label{fig:add-rel-effort-rf}
\end{figure}

\begin{figure}
    \centering
    \setlength{\tabcolsep}{0pt}
    \begin{tabular}{lc}
    & Uniformly-distributed perturbations\\
    \multirow{2}{*}{\rotatebox{90}{\rev{$f=\text{neural network}$}}} & 
    \includegraphics[width=0.98\linewidth]{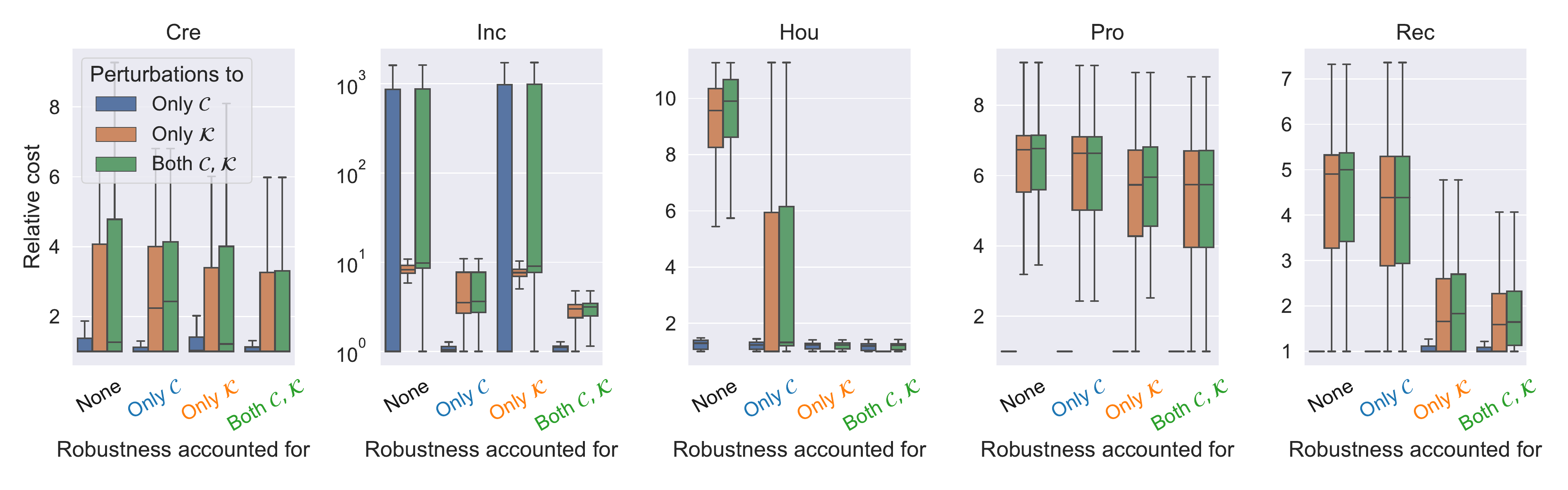}\\
    & Normally-distributed perturbations\\
    & \includegraphics[width=0.98\linewidth]{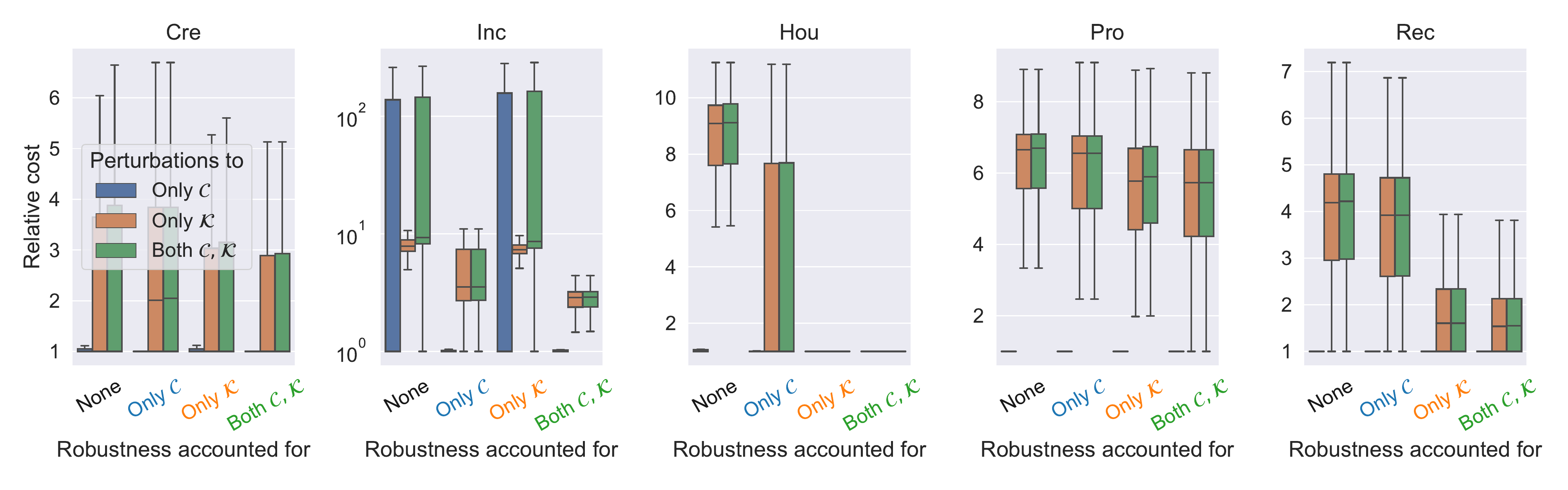}\\
    \end{tabular}
    \caption{
    Cost in terms of different configurations of accounting for robustness and under different perturbations, relative to the ideal cost \rev{(with neural network)}.
    Due to perturbations, the relative cost for when no notion of robustness is accounted for (label None) is typically much larger than the one for when the right notion of robustness is accounted for (matching color between box and label).
    The vertical axis for Inc is in logarithmic scale.}
    \label{fig:add-rel-effort-nn}
\end{figure}

\Cref{fig:add-rel-effort-rf,fig:add-rel-effort-nn} (\rev{for random forest and neural network, respectively)} show that when no robustness is accounted for (the left-most triplets of boxes in each plot), the relative cost can become dramatically large.
In other words, additional intervention to correct the perturbations can be extremely costly.
Whether the relative cost increases mostly due to perturbations to $\mathcal{C}$ (blue boxes) or to $\mathcal{K}$ (orange boxes) depends on the data set.
For example, perturbations to $\mathcal{K}$ have the largest effect on Rec, while those to $\mathcal{C}$ have the largest effect on Inc (by far)\rev{, across types of distribution and types of $f$}.
\rev{For both the random forest and the neural network, the relative cost ranges from around $5\times$ or $10\times$ the ideal cost, up to over $100\times$ (Inc, perturbations to $\mathcal{C}$) when not accounting for robustness.
}

\begin{figure}
    \setlength{\tabcolsep}{0pt}
    \begin{tabular}{lc}
    \centering
    \rotatebox{90}{\hspace{1cm}\rev{$f=\text{random forest}$}} & 
    \includegraphics[width=0.98\linewidth]{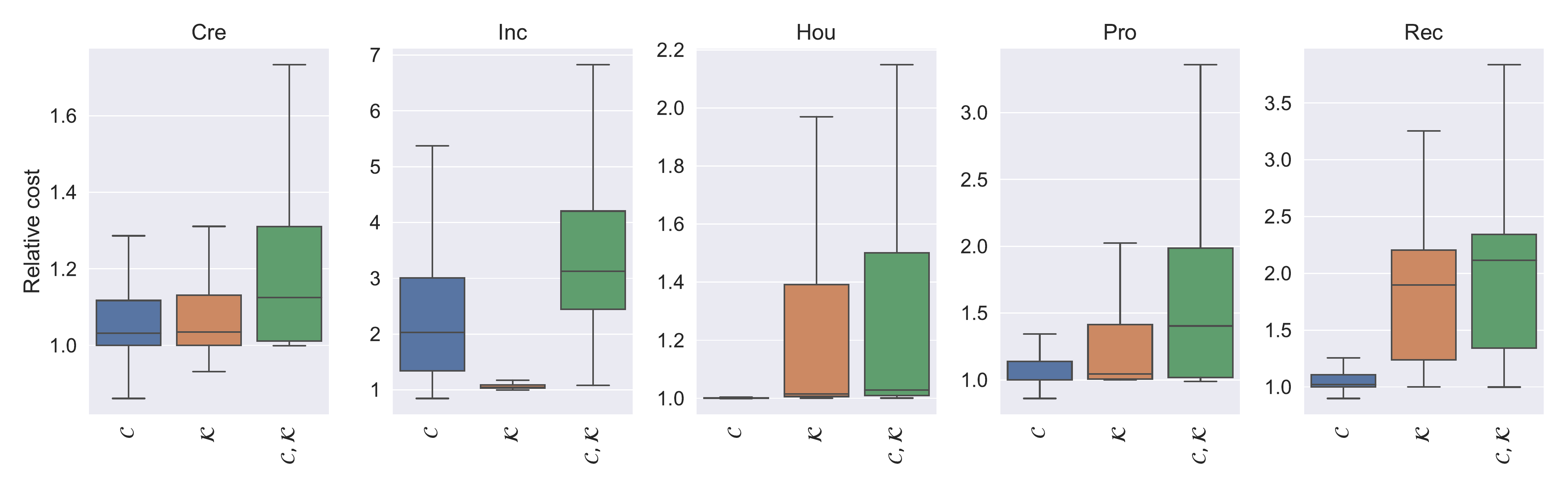}\\
    \rotatebox{90}{\hspace{1cm}\rev{$f=\text{neural network}$}} & \includegraphics[width=0.98\linewidth]{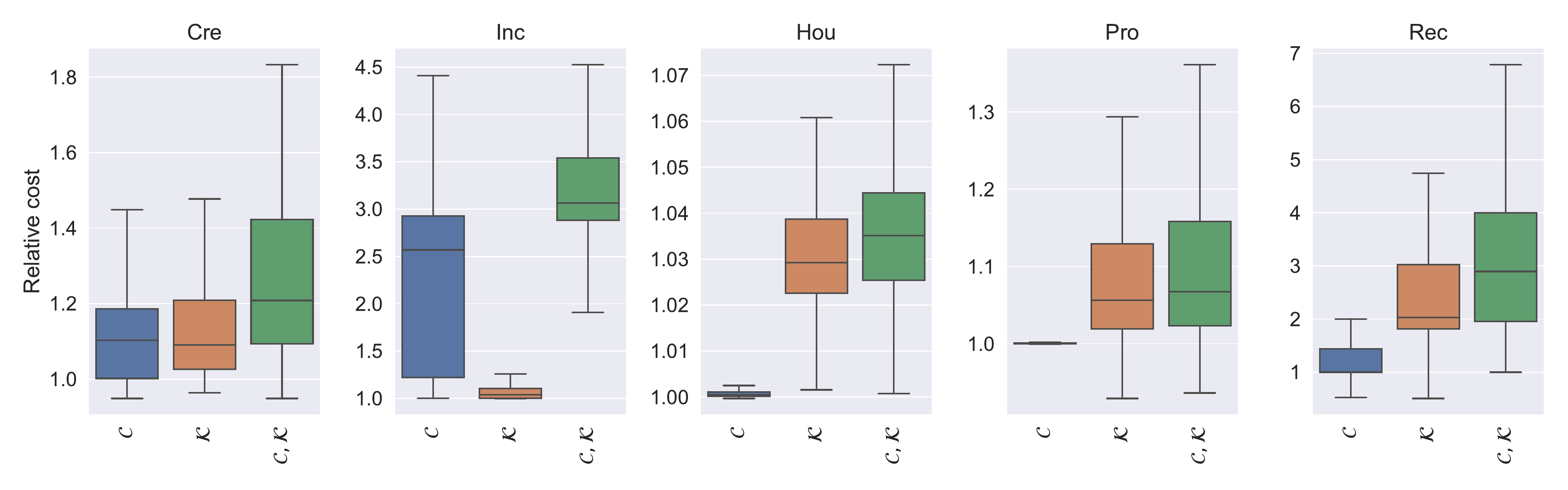}\\
    \end{tabular}
    \caption{
    \rev{Cost of accounting for robustness relative to not accounting for robustness (i.e., ideal cost) when no perturbations take place. 
    Note that the (rare) relative costs smaller than $1$ are due to a lack of optimality of the search algorithm.}}
    \label{fig:add-rel-effort-nop}
\end{figure}

When one accounts for the notion of robustness that is meant to deal with the respective type of perturbation, the relative cost often decreases substantially. 
Accounting for $\mathcal{C}$-robustness (second blue box from the left in each plot) counters perturbations to $\mathcal{C}$ very well \rev{on all the data sets}.
On Inc in particular, the relative cost improves by two orders of magnitude.
\rev{As found in \Cref{sec:results-rq2}, accounting for perturbations to $\mathcal{K}$ with the $\mathcal{K}$-robustness score can remain insufficient, as it can be observed on Cre and Inc for both types of $f$.}
Again, this is likely a limitation of using a simple heuristic such as the $\mathcal{K}$-robustness score to deal with $\mathcal{K}$-robustness.
Accounting for robustness w.r.t.~$\mathcal{C}$ (resp., $\mathcal{K}$) does not, in general, lead to smaller relative cost under perturbations to $\mathcal{K}$ (resp., $\mathcal{C}$). 
\rev{We confirm this general trend with statistical testing in \ref{sec:apdx-stat-sign}.}
Lastly, accounting for both $\mathcal{C}$- and $\mathcal{K}$-robustness (right-most triplets of boxes in each plot) offers protection (lower relative cost) from situations in which both types of perturbations take place.
\rev{In general across data sets and types of $f$}, the distribution of relative costs for when perturbations to both $\mathcal{C}$ and $\mathcal{K}$ take place and both $\mathcal{C}$- and $\mathcal{K}$-robustness are accounted for (right-most green box in each plot) is better than the distribution for when the same perturbations take place but no notion of robustness is accounted for (left-most green box in each plot).

\rev{Lastly,} since the ideal cost is computed when no notion of robustness is accounted for, part of the relative costs for when robustness is accounted for comes from the fact that robust counterfactual examples are generally farther away from $\mathbf{x}$ than non-robust ones.
\rev{\Cref{fig:add-rel-effort-nop} shows the cost increase that comes solely from accounting for robustness on the considered data sets and types of $f$, without any perturbation taking place}.
\rev{We remark that values smaller than \num{1} happen only because the discovered counterfactual examples can be suboptimal.}
\rev{Importantly, we find that the cost when  accounting for robustness is between $1\times$ and $7\times$ the ideal cost, i.e., when not accounting for robustness.
In general, this is significantly smaller than the increase incurred when perturbations take place and robustness is not accounted for, as reported before (generally between $5\times$ and $10\times$ the ideal cost, with up to $100\times$).
}

These results confirm that even though robust counterfactual explanations are, in principle, more costly to pursue than non-robust ones, if random perturbations take place, robust counterfactual explanations require much less additional intervention than non-robust ones.

\section{Discussion}
\label{sec:discussion}

Our experimental results provide a positive answer to all three research questions. 
Most often than not, counterfactual explanations are \emph{not} robust, be it in terms of the features whose value is prescribed to be changed ($\mathcal{C}$-robustness), or those whose value is prescribed to be kept as is ($\mathcal{K}$-robustness).
Moreover, non-robust counterfactual explanations are more susceptible to make it impossible for the user to remedy perturbations by additional intervention, and the cost of additional intervention is larger for non-robust counterfactual explanations than for robust ones.
Ultimately, it is clear that accounting for robustness is important.

Our experimental results \rev{suggest that accounting for robustness for features in $\mathcal{C}$ results in countering well perturbations to $\mathcal{C}$, and similarly, accounting for robustness for features in $\mathcal{K}$ results in countering well perturbations to $\mathcal{K}$.
Moreover, even though $f$ can learn non-linear feature interactions, accounting for $\mathcal{C}$ (or $\mathcal{K}$) has limited effect on contrasting perturbations to $\mathcal{K}$ (resp., $\mathcal{C}$).
Only in some cases (e.g., on Hou), robustness w.r.t.\ $\mathcal{C}$ has substantial repercussions on the effect of perturbations to $\mathcal{K}$ or vice versa.}

\rev{In addition to this, even if a counterfactual search algorithm does not guarantee that the discovered counterfactual example will be optimal, we experimentally see that incorporating our \Cref{def:c-robust-cfe} into the loss (\Cref{sec:incorp-rob}) produces a strong resilience to additional cost (\Cref{sec:results-rq3}) for perturbations to the features in $\mathcal{C}$.
Besides being effective, implementation of \Cref{def:c-robust-cfe} is also efficient (see \ref{sec:apdx-additional-runtime-robustness}).
}

\rev{What our results also show is that seeking robustness with respect to features in $\mathcal{K}$ is problematic.
This is because of \Cref{prop:general-f} and the fact that features in $\mathcal{K}$ are not aligned with the direction of intervention}.
Thus, we proposed to control for $\mathcal{K}$-robustness using an approximation, i.e., the $\mathcal{K}$-robustness score.
We found that seeking counterfactual examples that maximize the $\mathcal{K}$-robustness score is often but not always sufficient to obtain a good resilience to perturbations to the features in $\mathcal{K}$.
\rev{Moreover, the $\mathcal{K}$-robustness score requires to sample (and evaluate with $f$) multiple points, which is far more expensive than computing \Cref{def:c-robust-cfe}.}
Therefore, future work should consider whether a better method can be used than the $\mathcal{K}$-robustness score.
\rev{For example, if} information on $f$ is available, that information may be used to provide guarantees on the neighborhood of $\mathbf{z}$ (see, e.g., Theorem 2 in~\cite{dominguez2021adversarial} for linear $f$).

\rev{The assumption that features are independent is simplistic but often made in literature, because only a small number of works assume a causal model is available (e.g.,~\cite{karimi2020algorithmic,karimi2021algorithmic}).
Under the assumption of feature independence, as done here, one models the neighborhood of a counterfactual example with a box (under $L1$) or a hyper-sphere (under $L2$).
However, if certain features have a causal dependency on other features, this neighborhood morphs into other, possibly very complex shapes (e.g., when this dependency is not linear).
Importantly, if feature $i$ depends on $j$, then one cannot change $j$ without having that $i$ implicitly changes too.
Similarly, a perturbation happening to $j$ would implicitly alter $i$.
As our framework currently assumes independence, it is important to study to what extent separation between $\mathcal{C}$ and $\mathcal{K}$ remains possible and meaningful.
For many real-world problems, it is reasonable to expect that there exist groups of features that are truly independent from other groups of features.
Thus, the study of robustness for $\mathcal{C}$ and $\mathcal{K}$ could be carried out at a higher level, i.e., of feature groups in future work.
}

There is a number of further aspects worth mentioning when one wishes to implement a research work like on counterfactual explanations into practice, including this work.
For example, we use the $L1$-norm within Gower's distance to measure intervention cost.
In fact, literature works typically choose one distance measure (e.g., ours, or Gower's with $L2$-norm instead, or other variants, see~\Cref{sec:related}).
Of course, realistic implementation of intervention cost needs to use the a refined distance which may include mixing different types of norms, based on the features at play.
Similarly, one might wish to use different distributions to sample perturbations from (as opposed to only uniform or only normal as done in our synthetic experiments), \rev{and different functions to define the maximal extent of perturbation, which may e.g., account for the distribution of feature values.
For example, for denser areas of feature $i$, $p^+_i$ and $p^-_i$ should be smaller than for less dense areas.
Other desiderata may need to be included when seeking counterfactuals in practice (see, e.g.,~\cite{dandl2020multi,laugel2019unjustified}), including accounting for multiple types of robustness of the same time, such as those related to uncertainties of $f$~\cite{pawelczyk2020counterfactual,rawal2021algorithmic}
}


Lastly, we made subjective choices to define perturbations ($\mathbf{p}$) and plausibility constraints ($\mathcal{P}$) in the data sets.
We made these choices as best as we could, based on reading the meta-information in web sources and the papers that describe the data sets.
We have no doubt that domain experts would make much better choices than ours.
Nevertheless, we argue that this is not an important limitation because, as long as the community agrees that our choices are reasonable, they suffice to provide \rev{a sensible test bed for benchmarking robustness.
Hopefully, other researchers will find our annotations to be useful for future experiments on the robustness of counterfactual explanations.
Similarly, we hope that other researchers will find CoGS to be an interesting algorithm to benchmark against}.

\section{\rev{Related work}}\label{sec:related}

A number of works in literature propose several new desiderata that are largely orthogonal to our notions of robustness but can be important to enhance the practical usability of counterfactual explanations.
For example, Dandl et al.~\cite{dandl2020multi} consider, besides proximity of $\mathbf{z}$ to $\mathbf{x}$ according to different distances, whether other training points $\mathbf{x}^\prime$ are sufficiently close to $\mathbf{z}$ for it to reasonably belong to the training data distribution.
A similar desideratum is considered in~\cite{sharma2020certifai} and~\cite{van2019interpretable}; 
the latter work employs neural autoencoders to that end.
\rev{\cite{dhurandhar2018explanations} remarks the importance of sparsity for explanations, with the concepts of \emph{pertinent negatives} (the minimal features that should be different to (more) confidently predict the given class) and \emph{pertinent positives} (the minimal features that help correctly identifying the class).
}
Laugel et al.~\cite{laugel2019dangers,laugel2019unjustified} require that $\mathbf{z}$ can always be reached from a training point $\mathbf{x}^\prime$ without having to cross the decision boundary of $f$, for $\mathbf{z}$ not to be the result of an artifact in the decision boundary of $f$.
In~\cite{karimi2020algorithmic} and~\cite{mothilal2020explaining}, counterfactual explanations are studied through the lens of causality.
For recent surveys on counterfactual explanations, the reader is referred to~\cite{verma2020counterfactual,stepin2021survey,karimi2021survey}.

We now focus on works that deal with some notion of robustness and/or perturbations explicitly.
\rev{Artelt et al.~\cite{artelt2021evaluating} present theoretical results on the effect of perturbations (e.g., under linear $f$), evaluate the effect of different type of perturbations (Gaussian, uniform, masking) with three classifiers, and find that counterfactual explanations that obey plausibility constraints are more robust than counterfactual explanations that do not. 
Differently from us, Artelt et al.~do not consider sparsity and do not optimize for robustness.
}
The work by Karimi et al.~\cite{karimi2021algorithmic} extends~\cite{karimi2020algorithmic} to consider possible uncertainties in causal modelling.
In~\cite{slack2021counterfactual}, it is shown that a malicious actor can, in principle, jointly optimize small perturbations and the model $f$ such that, when applying the perturbations to points of a specific group (e.g., white males), the respective counterfactual explanations are much less costly than normal (in fact, counterfactual explanations are conceptually similar to adversarial examples, see, e.g.,~\cite{pawelczyk2021exploring,ballet2019imperceptible,freiesleben2021intriguing}).
Some works consider forms of robustness of counterfactual explanations with respect to changes of $f$ (e.g., whether $\mathbf{z}$ is still classified as $t$ if $f^\prime$ is used instead of $f$)~\cite{pawelczyk2020counterfactual,ferrario2022robustness} or updates to $f$ (e.g., after data distribution shift of temporal or geospatial nature)~\cite{ferrario2020series,rawal2021algorithmic}.
In~\cite{mochaourab2021robust}, robustness of counterfactual explanations is studied in the context of differentially-private support vector machines.
Dominguez et al.~\cite{dominguez2021adversarial} consider whether counterfactual explanations remain valid in presence of uncertainty on $\mathbf{x}$, and also account for causality.
\rev{We also note that Dominiguez et al.~consider a neighborhood of uncertainty around $\mathbf{x}$ which is akin to \Cref{def:p-neighborhood}; in fact, such sort of neighborhoods are common tools in post-hoc explanation methods, e.g., the Anchor explainer by \cite{ribeiro2018anchors} seeks representative points for a class by assessing that the prediction of $f$ for the points in their neighborhood is the same.}
\rev{Zhang et al.~\cite{zhang2018interpreting} propose a counterfactual search method based on linear programming that works for neural networks with ReLU activations; this work can be seen through the lens of robustness in that the method produces \emph{regions} of points that share the desired class.}
\rev{Finally, contemporary to our work, Fokkema et al.~\cite{fokkema2022attribution} provide important theoretical results that counterfactual explanations (and other XAI methods such as feature attribution ones) can be dramatically different when small perturbations are applied to the starting point $\mathbf{x}$ (or, in general, point to explain).}

To the best of our knowledge, there exists no other work prior to ours that attempts to exploit sparsity when assessing robustness, although sparsity is an important property for counterfactual explanations.
Moreover, existing works typically consider whether robustness helps preventing counterfactual explanations from becoming invalid, while we further consider that additional intervention may be possible, and assess the associated cost.

\section{Conclusion}
Counterfactual explanations can help us understand how black-box AI systems reach certain decisions, as well as what intervention is possible to alter such decisions.
For counterfactual explanations to be most useful in practice, we studied how they can be made \emph{robust} to adverse perturbations that may naturally happen due to unfortunate circumstances, to ensure that the intervention they prescribe remains valid, and potential additional intervention cost that may be needed remains limited.
We presented novel notions of robustness, which concern adverse perturbations to the features that a counterfactual explanation prescribes to change ($\mathcal{C}$-robustness) and to keep as they are ($\mathcal{K}$-robustness), respectively.
We have annotated five existing data sets with reasonable perturbations and plausibility constraints and developed a competitive counterfactual search algorithm to search for (robust) counterfactual explanations.
Our experimental results show that, most often than not, counterfactual explanations do not happen to be robust by accident. 
Consequently, if adverse perturbations take place, counterfactual explanations may require a much larger cost to be realized than anticipated, or even make it impossible for the user to achieve recourse.
Our definitions of robustness can be incorporated in the search process, and robust counterfactual explanations can be discovered.
We have shown that $\mathcal{C}$-robustness can be accounted for efficiently and effectively, while the same is not always true for $\mathcal{K}$-robustness.
\rev{This aspect should be taken into account when choosing what counterfactual explanation is best for the user.}
Overall, robust counterfactual explanations are resilient against invalidation and require much smaller additional intervention to contrast perturbations.

\section*{Acknowledgments}
We thank dr.~Stef C.~Maree for insightful early discussions.
This work made use of the Dutch national e-infrastructure with the support of
the SURF Cooperative using grant no. EINF-2512.
Funding: This publication is part of the project Robust Counterfactual Explanations (with project number EINF-2512) of the research program Computing Time on National Computer Facilities which is (partly) financed by the Dutch Research Council (NWO).

\bibliographystyle{elsarticle-num}
\bibliography{main}

\appendix

\section{Hyper-parameter optimization of random forest \rev{and neural network}}
\label{sec:apdx-hyperparams-random-forest}
To obtain a black-box model $f$ for a given cross-validation fold, we train a random forest model \rev{or a neural network (a multi layer perceptron for classification)} optimized with grid-search hyper-parameter tuning (with five-fold cross-validation on the training set).
The hyper-parameter settings we considered are listed in \Cref{tab:hyper-params-rf,tab:hyper-params-nn}, \rev{all other being Scikit-learn's default~\cite{pedregosa2011scikit} (v.~1.0.1)}.
For random forest, we one-hot encode categorical features when training and querying the random forest model (see the code \texttt{robust\_cfe/blackbox\_with\_preproc.py}).
\rev{For the neural network, we additionally scale numerical features to have mean of zero and standard deviation of one.}

The performance of tuned random forest on all folds in shown in \Cref{tab:performance-random-forest}, \rev{the respective one for the neural network is shown in \Cref{tab:performance-neural-net}}.

\begin{table}[h]
    \centering
    \caption{Hyper-parameter settings considered for tuning random forest.}
    \begin{tabular}{lc}
    \toprule
        Name & \rev{Options} \\
    \midrule
        No.~trees & $\{50, 500\}$ \\
        Min.~samples~split & $\{2, 8\}$\\
        Max.~features & \{$\sqrt{d}$, $d$ \} \\
    \bottomrule
    \end{tabular}
    \label{tab:hyper-params-rf}
\end{table}

\begin{table}[h]
    \centering
    \caption{\rev{Hyper-parameter settings considered for tuning the neural network}.}
    \begin{tabular}{lc}
    \toprule
        Name & Options \\
    \midrule
        Learning rate & $\{0.0001, 0.01\}$ \\
        Max.~iterations & $\{200,  1000\}$ \\
        Solver & $\{\text{Adam}, \text{SGD}\}$\\
    \bottomrule
    \end{tabular}
    \label{tab:hyper-params-nn}
\end{table}

\begin{table}[h]
    \centering
    \caption{Test accuracy of hyper-parameter-tuned random forests acting as black-box models $f$ for the considered data sets across five-fold cross-validation.} 
    \sisetup{separate-uncertainty,scientific-notation=fixed,round-mode=places}
    \begin{tabular}{cccccc}
    \toprule
    {Fold} & {Cre} & {Inc} & {Hou} & {Pro} & {Rec} \\
    \midrule
    0 & 0.71 & 0.86 & 0.93 & 0.79 & 0.8  \\ 
    1 & 0.78 & 0.82 & 0.9 & 0.77 & 0.82  \\ 
    2 & 0.78 & 0.79 & 0.91 & 0.78 & 0.78  \\ 
    3 & 0.74 & 0.82 & 0.91 & 0.82 & 0.77  \\ 
    4 & 0.76 & 0.83 & 0.97 & 0.78 & 0.8  \\ 
    \hline
    Avg. & 0.76 & 0.83 & 0.93 & 0.79 & 0.8  \\ 
    \bottomrule
    \end{tabular}
    \label{tab:performance-random-forest}
\end{table}

\begin{table}[h]
    \centering
    \caption{\rev{Test accuracy of hyper-parameter-tuned neural networks acting as black-box models $f$ for the considered data sets across five-fold cross-validation}.} 
    \sisetup{separate-uncertainty,scientific-notation=fixed,round-mode=places}
    \begin{tabular}{cccccc}
    \toprule
    {Fold} & {Cre} & {Inc} & {Hou} & {Pro} & {Rec} \\
    \midrule
    0 & 0.74 & 0.83 & 0.94 & 0.62 & 0.78  \\ 
    1 & 0.78 & 0.82 & 0.93 & 0.7 & 0.79  \\ 
    2 & 0.73 & 0.8 & 0.91 & 0.69 & 0.75  \\ 
    3 & 0.78 & 0.82 & 0.91 & 0.77 & 0.78  \\ 
    4 & 0.74 & 0.82 & 0.96 & 0.72 & 0.79  \\ 
    \hline
    Avg. & 0.75 & 0.82 & 0.93 & 0.7 & 0.78  \\ 
    \bottomrule
    \end{tabular}
    \label{tab:performance-neural-net}
\end{table}

\section{Additional results}
\label{sec:apdx-additional-results}

We provide additional results.
These are the (possibly non-permanent) invalidity caused by perturbations, as typically done in the literature of robustness, and effect of increasing $m$ for the computation of the $\mathcal{K}$-robustness score.

\subsection{Invalidity of counterfactual explanations}
\label{sec:apdx-invalidity}
We now show whether the fact that best-found counterfactual explanations are typically not robust is associated with a greater chance that perturbations can make them invalid, i.e., such that $f(\mathbf{z}^\prime) \neq t$ where $\mathbf{z}^\prime$ is the point to which $\mathbf{z}^\star$ is shifted by the perturbation.
Here, we do not consider whether further intervention may or may not be possible.
\Cref{fig:invalidity} shows the average frequency with which perturbations cause invalidity.
The frequencies are computed by applying, to each discovered counterfactual example, \num{100} perturbations that are sampled uniformly at random for categorical features (from the categorical possibilities) and uniformly or normally (with st.dev. of \num{0.1}) for numerical features (within the numerical intervals).
The figure shows that when no notion of robustness is accounted for, perturbations generally have a larger chance of causing invalidity of the counterfactual explanation.

Regarding $\mathcal{C}$-robustness and respective perturbations to (features in) $\mathcal{C}$ (i.e., $\mathcal{C}$-setbacks), recall that accounting for this notion of robustness is intended to provide counterfactual explanations with minimal additional intervention cost, the maximal $\mathcal{C}$-setback were to happen.
However, \rev{ideally}, the returned counterfactual example should still be near $\mathbf{x}$, i.e., it should be a point on the boundary of $f$ (exactly so, if the point is truly optimal). 
Thus, under optimality guarantees, $f(\mathbf{z}^\star + \mathbf{w}^{c,s}) \neq t$ (\Cref{prop:optimal-c}); 
This means that any $\mathcal{C}$-setback should result in invalidity (i.e., all entries for perturbations to $\mathcal{C}$ should report \num{1}).
\rev{This does not always happen in \Cref{fig:invalidity} because CoGS does not guarantee to discover optimal counterfactual examples and, thus, in many cases the returned example is not on the boundary, and the $\mathcal{C}$-setback is too small to cross the boundary.
The frequency of this phenomenon depends on the data set.
Also, while accounting for robustness w.r.t.~$\mathcal{C}$ should not, in theory, decrease invalidity rate but only make further intervention less costly, as confirmed in \Cref{sec:results-rq3}), we find that accounting for robustness w.r.t.~$\mathcal{C}$ lowers invalidity rate on Hou (e.g., most evident for both types of $f$ with normally-distributed perturbations).
}

When $\mathcal{K}$-robustness is accounted for, the best-found counterfactual explanation is supposed to be in a region such that the decision boundary is relatively loose with respect to the features in $\mathcal{K}$.
Consequently, accounting for $\mathcal{K}$-robustness \emph{should}, \rev{in fact, counter invalidity, as we do not wish risking that it becomes impossible to carry out further intervention due to the plausibility constraints.}
\rev{The figure shows that, in general, there can be a substantial gain in lowering invalidity by accounting for $\mathcal{K}$-robustness.}
At times, \rev{accounting for $\mathcal{K}$-robustness allows to reach almost zero invalidity}, see \rev{the cell that corresponds to robustness for $\mathcal{K}$ and perturbations to $\mathcal{K}$, on} Inc, Hou, and Pro, for both types of $f$ and sampling distributions.
\rev{However, it is not always the case that $\mathcal{K}$-robustness helps, due to the heuristic nature of the $\mathcal{K}$-robustness score: see, e.g., Cre}.

Lastly, we observe that the frequency of invalidity can raise when both notions of robustness are accounted for at the same time (e.g., on Inc \rev{for uniformly-distributed perturbations when using the neural network}).
\rev{Note that this is not necessarily a problem because invalidity from perturbations to $\mathcal{C}$ is expected to be high, as the goal of robustness w.r.t.~$\mathcal{C}$ is to be able to minimize additional intervention cost.}

\begin{figure}[h]
    \centering
    \setlength{\tabcolsep}{0pt}
    \begin{tabular}{lc}
    &
    Uniformly-distributed perturbations\\
    \multirow{2}{*}{\rev{\rotatebox{90}{$f=\text{random forest}$}}} 
    &
    \includegraphics[width=\linewidth]{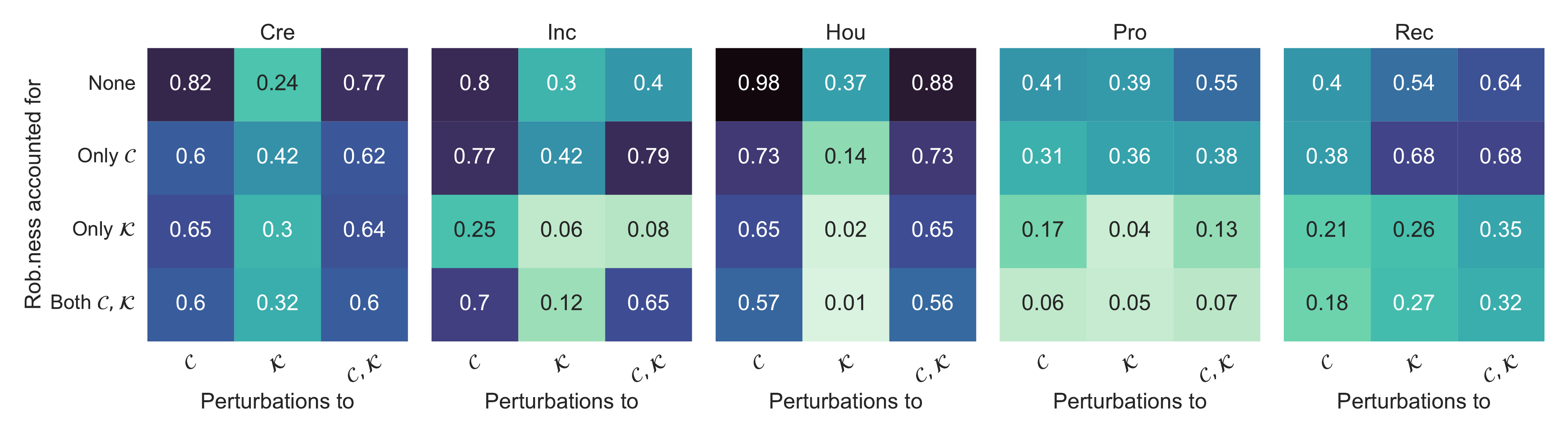}\\ 
    &
    Normally-distributed perturbations\\
    &
    \includegraphics[width=\linewidth]{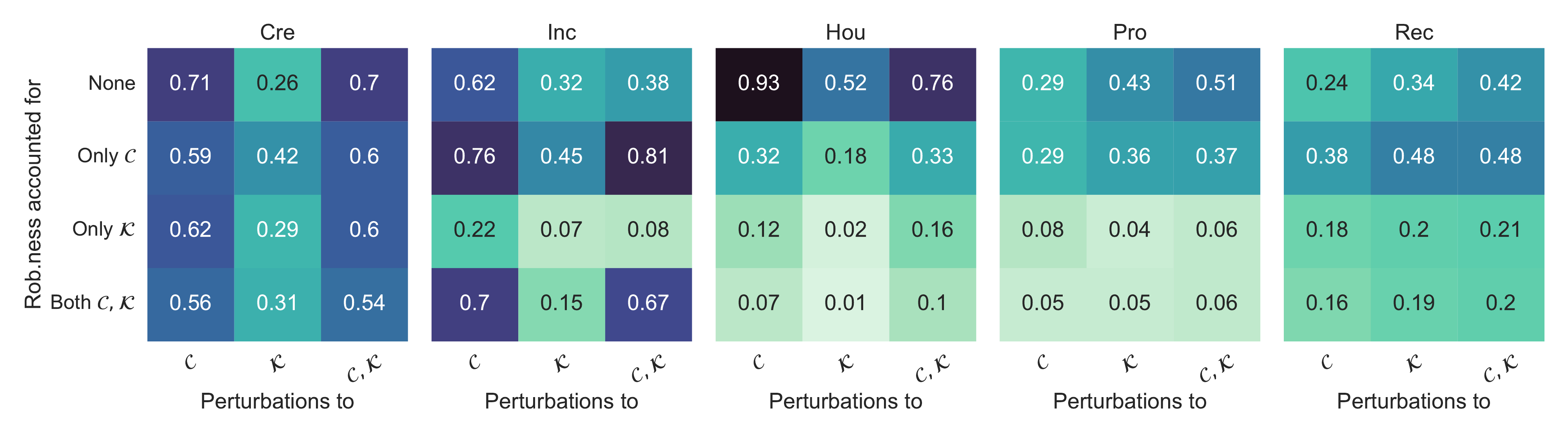}\\ 
    &
    Uniformly-distributed perturbations\\
    \multirow{2}{*}{\rev{\rotatebox{90}{$f=\text{neural network}$}}} 
    &
    \includegraphics[width=\linewidth]{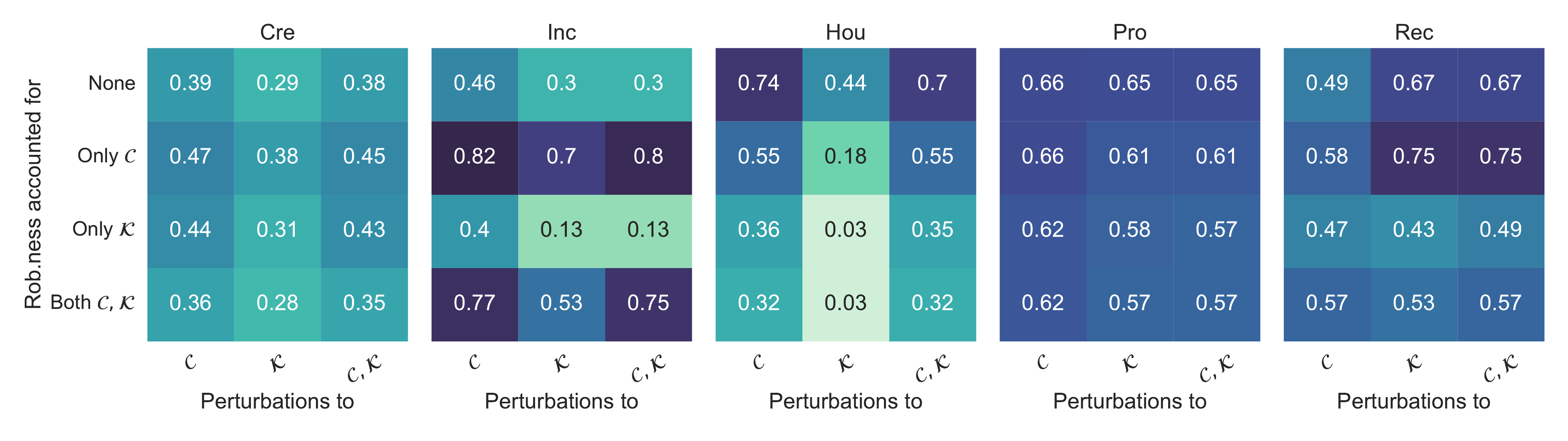}\\ 
    &
    Normally-distributed perturbations\\
    &
    \includegraphics[width=\linewidth]{pics_rev1/invalidity_normal_rf.pdf}\\ 
    \end{tabular}
    \caption{Mean frequency of invalidity of counterfactual explanations under different types of perturbations and when accounting for different types of robustness.
    Darker colors represent worse scenarios, i.e.,  larger average invalidity.}
    \label{fig:invalidity}
\end{figure}

\subsection{Setting $m$ for $\mathcal{K}$-robustness}
\label{sec:apdx-sensitivity-k-robustness}

We report results on setting the hyper-parameter $m$ for computing $\mathcal{K}$-robustness scores (see \Cref{eq:practical-k-robust}).
In particular, we run CoGS accounting for $\mathcal{K}$-robustness in the loss function, for $m \in \{0, 4, 16, 64\}$.
Note that using $m=0$ corresponds to \emph{not} accounting for $\mathcal{K}$-robustness.

\subsubsection{Achieved $\mathcal{K}$-robustness}
We consider how increasing $m$ improves $\mathcal{K}$-robustness, using an \rev{an approximated \emph{ground-truth}}.
\rev{We approximate the ground-truth of the true $\mathcal{K}$-robustness by calculating the $\mathcal{K}$-robustness score over} \num{1000} samples \rev{over the counterfactual example discovered using a specific $m$}.

\Cref{fig:increasing-m} shows the results obtained for this experiment.
We also consider the case in which $\mathcal{C}$-robustness is accounted for.
If $\mathcal{K}$-robustness is not accounted for ($m=0$), then the (approximated ground-truth) $\mathcal{K}$-robustness of the discovered counterfactual examples \rev{can be quite low, see, e.g., Rec for random forest (score approximately of $0.4$) and neural network (score below $0.4$)}.
As soon as a few samples are considered ($m=4$), the $\mathcal{K}$-robustness increases substantially (see, e.g., Cre).
Further increasing $m$ has diminishing returns (note that $m$ is increased exponentially).
Accounting for $\mathcal{C}$-robustness is largely orthogonal, meaning, it has no effect in terms of $\mathcal{K}$-robustness.

\begin{figure}[h]
    \centering
    $f=\text{random forest}$\\
    \includegraphics[width=\linewidth]{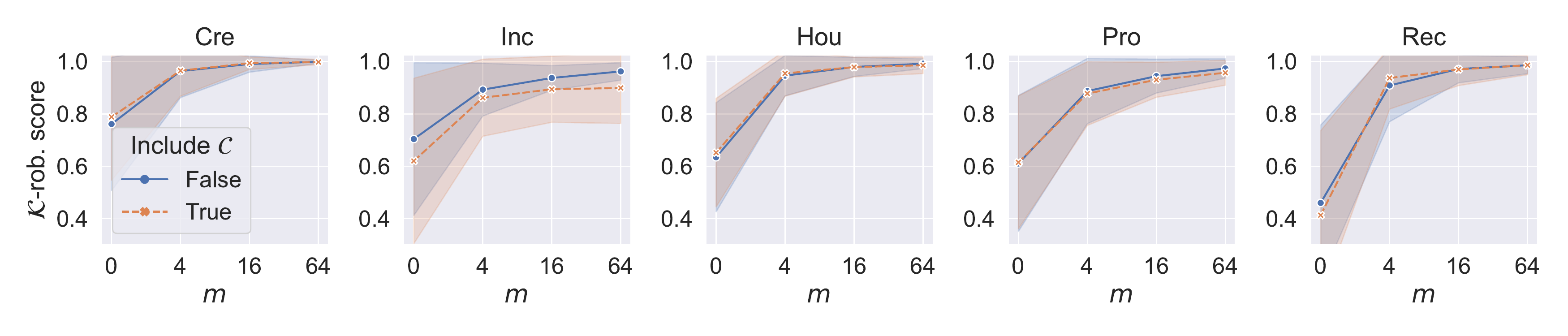}\\
    $f=\text{neural network}$\\
    \includegraphics[width=\linewidth]{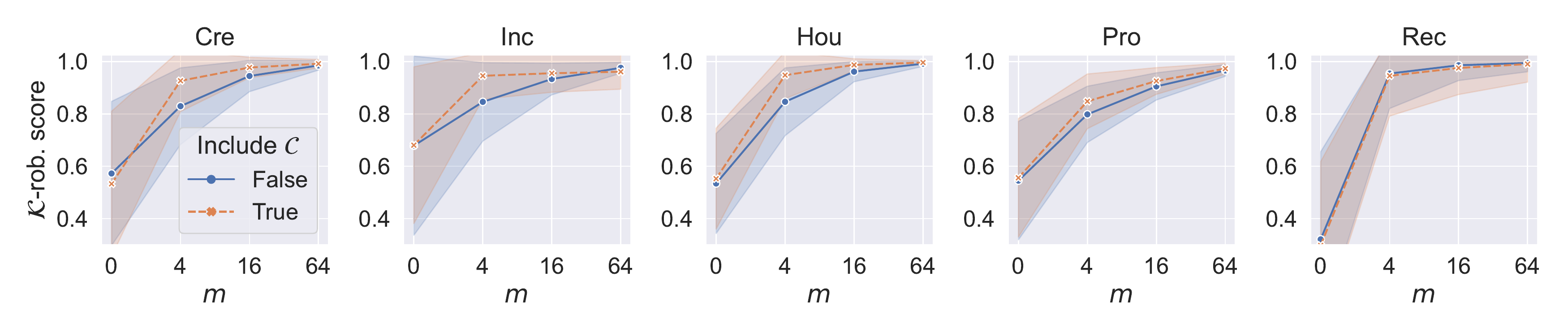}
    \caption{Approximated ground-truth $\mathcal{K}$-robustness scores (using $1000$ samples) for increasing $m$, to determine what value of $m$ is needed for good robustness to perturbations to  $\mathcal{K}$. 
    Shaded areas represent standard deviations.}
    \label{fig:increasing-m}
\end{figure}

\subsubsection{Additional required runtime}
\label{sec:apdx-additional-runtime-robustness}
\Cref{fig:cost-robustness-runtime} shows the additional runtime incurred between runs of CoGS that account for some notion of robustness and runs that do not account for it, \rev{in particular for increasing $m$ in the calculation of the $\mathcal{K}$-robustness score}.
The figure shows that accounting for $\mathcal{C}$-robustness comes at no significant extra cost in runtime.
This follows from the fact that we can use \Cref{def:c-robust-cfe} and thus only need to compute the maximal $\mathcal{C}$-setback.
Conversely, accounting for $\mathcal{K}$-robustness can come at a relatively large additional cost in runtime, which appears to be linear in $m$ (note that $m$ grows exponentially in the plots).
Fortunately, the experimental results of \ref{sec:apdx-sensitivity-k-robustness} suggest that small values of $m$ are often sufficient to obtain good $\mathcal{K}$-robustness scores.

\begin{figure}[h]
    \centering
    \rev{$f=\text{random forest}$}\\
    \includegraphics[width=\linewidth]{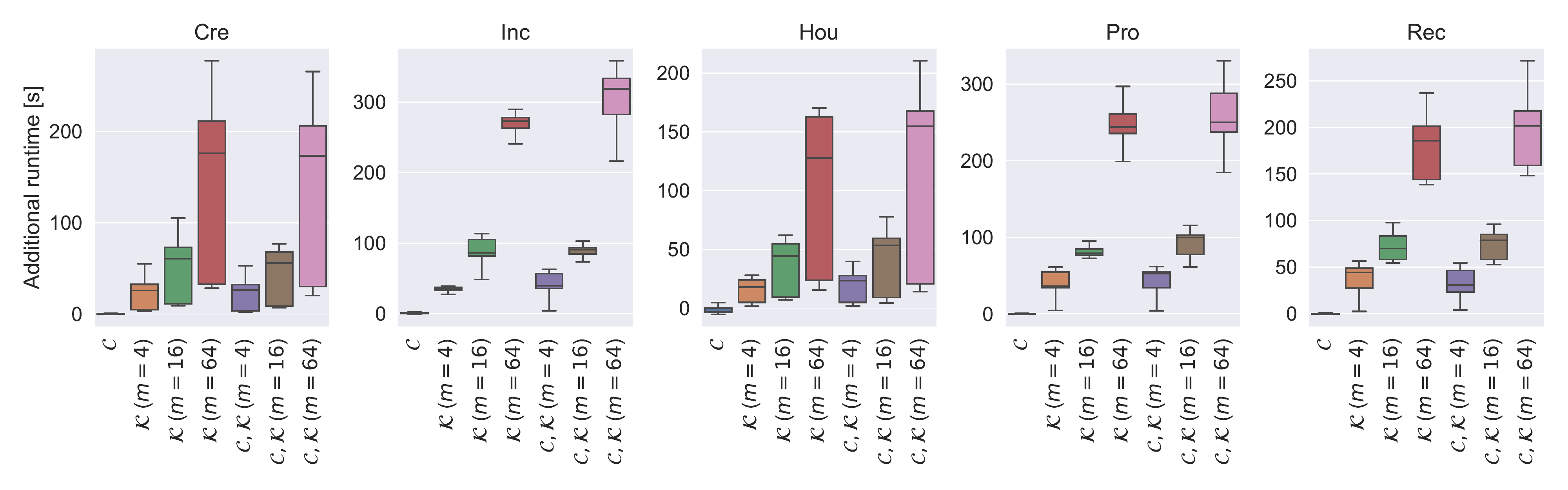}\\
    \rev{$f=\text{neural network}$}\\
    \includegraphics[width=\linewidth]{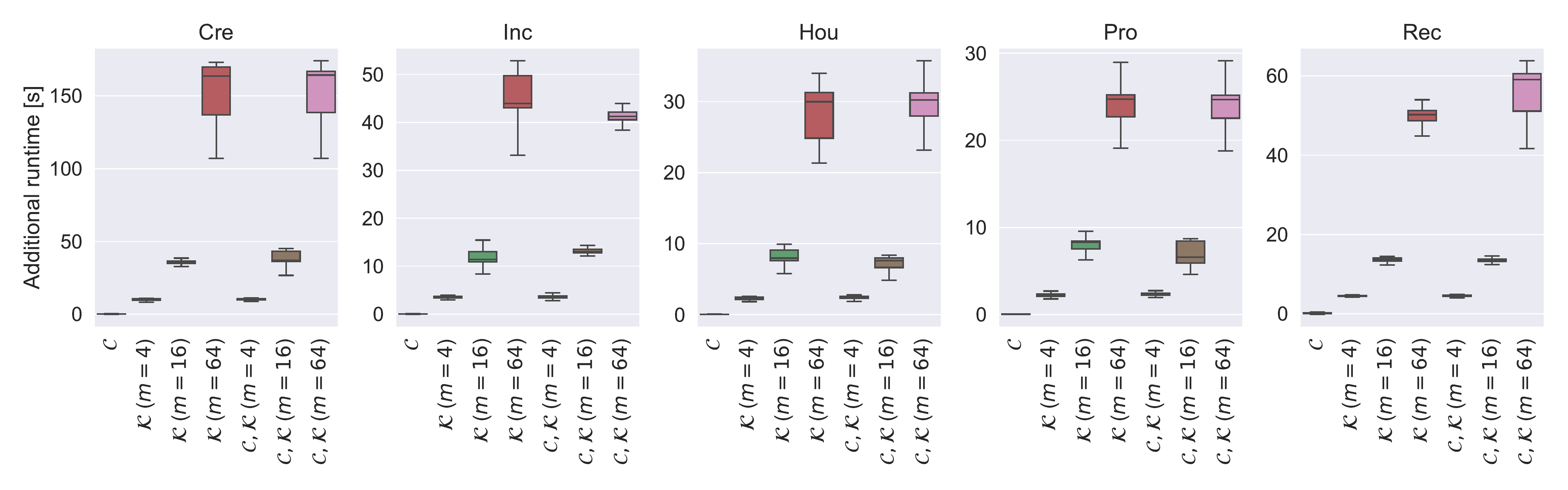}
    \caption{Additional runtime of CoGS for different configurations of accounting for robustness with respect to the runtime when not accounting for robustness.}
    \label{fig:cost-robustness-runtime}
\end{figure}

\subsubsection{Additional cost from accounting for robustness}
\label{sec:apdx-cost-robustness}
\rev{ \Cref{fig:cost-robustness-loss} expands on the results reported in \Cref{fig:add-rel-effort-nop} by including different values of $m$}.
We do not find major differences based on the setting of $m$ for computing the $\mathcal{K}$-robustness score, except for the tails of the respective distributions on Hou, and slightly less so on Pro \rev{(for both types of $f$)}.
Accounting for $\mathcal{C}$- and $\mathcal{K}$-robustness at the same time leads to larger costs than accounting for only one of the two, as it is reasonable to expect.
On average, the cost that comes from accounting for robustness alone is limited (up to $6.5\times$ the ideal cost, see Inc), especially in light of the results found for when perturbations take place, described in \Cref{sec:results-rq3} (additional intervention due to perturbations can lead to $100\times$ times larger costs for non-robust counterfactual explanations, see Inc on \Cref{fig:add-rel-effort-rf}).

\begin{figure}
    \centering
    \rev{$f=\text{random forest}$}\\
    \includegraphics[width=\linewidth]{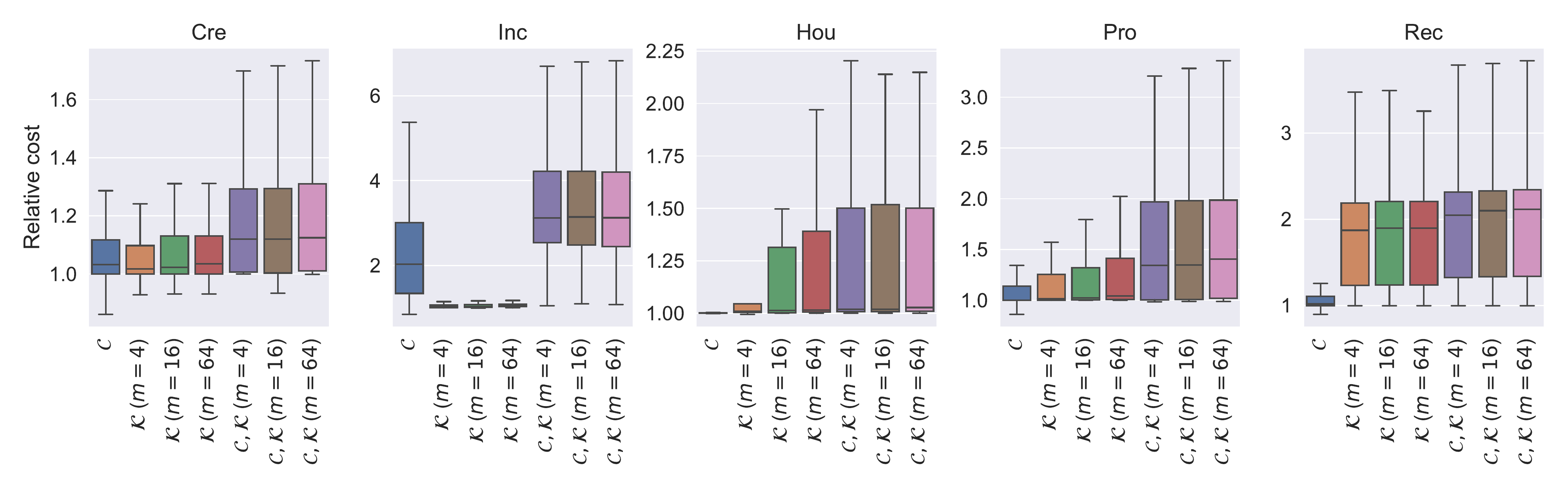}\\
    \rev{$f=\text{neural network}$}\\
    \includegraphics[width=\linewidth]{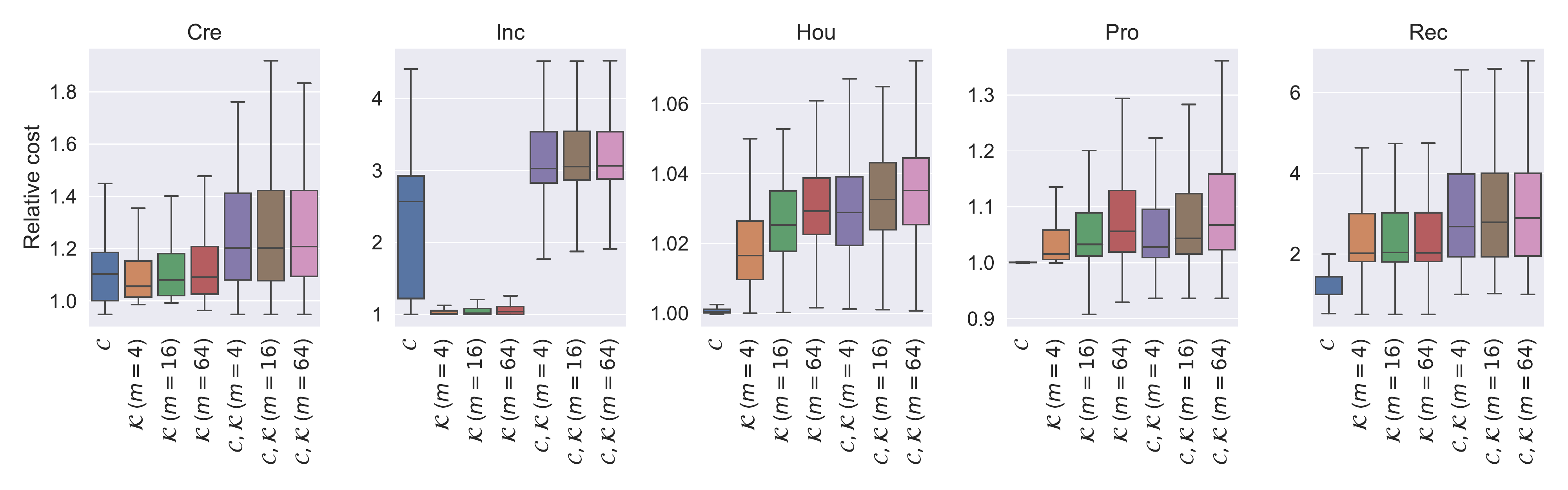}\\
    \caption{Cost of accounting for robustness relative to not accounting for robustness (i.e., ideal cost) when no perturbations take place for different values of $m$. 
    Note that the (rare) relative costs smaller than $1$ are due to a lack of optimality of the search algorithm.
    }
    \label{fig:cost-robustness-loss}
\end{figure}

\section{Statistical significance}
\label{sec:apdx-stat-sign}
We report the statistical significance for the results displayed in \Cref{sec:results-rq3}. 
For each data set and type of perturbation, we perform the Kruskall-Wallis tests (since we cannot assume normality) to determine whether significant differences are present between the relative cost induced by applying the different notion of robustness.
In all cases, the outcome of the test is that significant differences are present ($p\text{-value} \ll 0.01$).
Next, we perform post-hoc pairwise comparisons with the Mann-Whitney-U test to assess whether one notion of robustness protects from the considered perturbation significantly differently than another.
The result of the pairwise comparison analysis is shown in \Cref{tab:ss-cre,tab:ss-inc,tab:ss-hou,tab:ss-pro,tab:ss-rec}.

\rev{On Cre under $\mathcal{K}$-perturbations (middle part of \Cref{tab:ss-cre}), accounting for robustness w.r.t.\ $\mathcal{C}$ is not significantly different than not accounting for any notion of robustness ($p\text{-value}=0.52 > 0.01$);
similarly, accounting for robustness w.r.t.\ $\mathcal{K}$ induces the same relative cost as accounting for robustness w.r.t.\ both $\mathcal{C}$ and $\mathcal{K}$.
On Hou under $\mathcal{C}$-perturbations (top part of \Cref{tab:ss-hou}), accounting for $\mathcal{C}$ is not significantly different than accounting for $\mathcal{K}$ ($p\text{-value}=0.52 > 0.01$), while accounting for $\mathcal{K}$ is not significantly different than accounting for both $\mathcal{C}$ and $\mathcal{K}$ ($p\text{-value}=0.083 > 0.01$).
When perturbations happen to both $\mathcal{C}$ and $\mathcal{K}$ on Pro and Rec (bottom part of respective tables), accounting for $\mathcal{K}$ is not significantly different than accounting for both $\mathcal{C}$ and $\mathcal{K}$ ($p\text{-value}=0.014 > 0.01$ and $p\text{-value}=0.271 > 0.01$, respectively).
In general, the results match what can be seen in \Cref{fig:add-rel-effort-rf,fig:add-rel-effort-nn}.
Also, we note that in the majority of the cases, accounting for one notion of robustness \emph{is} significantly different than accounting for another (or for none).
}

\begin{table}[]
    \centering
    \caption{Result of pairwise comparison on the effect of accounting for different types of robustness for data set Cre under different perturbations (both random forest and neural network, both uniform and normal sampling distributions).
    The displayed $p$-values are obtained with the Mann-Whitney U test under Holm-Bonferroni correction and post Kruskal-Wallis test rejecting the null hypothesis with $p\text{-value} \ll 0.01$.}
\begin{tabular}{lcccc}
\toprule
 \multicolumn{5}{c}{Cre, perturbations to $\mathcal{C}$}\\
 Robustness & {None} & {Only $\mathcal{C}$} & {Only $\mathcal{K}$} & {Both $\mathcal{C}, \mathcal{K}$} \\ \midrule
{None} & 1.000 & 0.000 & 0.000 & 0.000 \\
{Only $\mathcal{C}$} & 0.000 & 1.000 & 0.000 & 0.125 \\
{Only $\mathcal{K}$} & 0.000 & 0.000 & 1.000 & 0.000 \\
{Both $\mathcal{C}, \mathcal{K}$} & 0.000 & 0.125 & 0.000 & 1.000 \\
\midrule
\midrule
 \multicolumn{5}{c}{Cre, perturbations to $\mathcal{K}$}\\
 Robustness & {None} & {Only $\mathcal{C}$} & {Only $\mathcal{K}$} & {Both $\mathcal{C}, \mathcal{K}$} \\ \midrule
{None} & 1.000 & 0.520 & 0.000 & 0.000 \\
{Only $\mathcal{C}$} & 0.520 & 1.000 & 0.000 & 0.000 \\
{Only $\mathcal{K}$} & 0.000 & 0.000 & 1.000 & 0.912 \\
{Both $\mathcal{C}, \mathcal{K}$} & 0.000 & 0.000 & 0.912 & 1.000 \\
\midrule
\midrule
 \multicolumn{5}{c}{Cre, perturbations to $\mathcal{C} \text{ and } \mathcal{K}$}\\
 Robustness & {None} & {Only $\mathcal{C}$} & {Only $\mathcal{K}$} & {Both $\mathcal{C}, \mathcal{K}$} \\ \midrule
{None} & 1.000 & 0.000 & 0.000 & 0.000 \\
{Only $\mathcal{C}$} & 0.000 & 1.000 & 0.000 & 0.000 \\
{Only $\mathcal{K}$} & 0.000 & 0.000 & 1.000 & 0.003 \\
{Both $\mathcal{C}, \mathcal{K}$} & 0.000 & 0.000 & 0.003 & 1.000 \\
\bottomrule
\end{tabular}
    \label{tab:ss-cre}
\end{table}

\begin{table}[]
    \centering
    \caption{Result of pairwise comparison on the effect of accounting for different types of robustness for data set Inc under different perturbations (both random forest and neural network, both uniform and normal sampling distributions).
    The displayed $p$-values are obtained with the Mann-Whitney U test under Holm-Bonferroni correction and post Kruskal-Wallis test rejecting the null hypothesis with $p\text{-value} \ll 0.01$.}
\begin{tabular}{lcccc}
\toprule
 \multicolumn{5}{c}{Inc, perturbations to $\mathcal{C}$}\\
 Robustness & {None} & {Only $\mathcal{C}$} & {Only $\mathcal{K}$} & {Both $\mathcal{C}, \mathcal{K}$} \\ \midrule
{None} & 1.000 & 0.000 & 0.000 & 0.000 \\
{Only $\mathcal{C}$} & 0.000 & 1.000 & 0.000 & 0.000 \\
{Only $\mathcal{K}$} & 0.000 & 0.000 & 1.000 & 0.000 \\
{Both $\mathcal{C}, \mathcal{K}$} & 0.000 & 0.000 & 0.000 & 1.000 \\
\midrule
\midrule
 \multicolumn{5}{c}{Inc, perturbations to $\mathcal{K}$}\\
 Robustness & {None} & {Only $\mathcal{C}$} & {Only $\mathcal{K}$} & {Both $\mathcal{C}, \mathcal{K}$} \\ \midrule
{None} & 1.000 & 0.000 & 0.000 & 0.000 \\
{Only $\mathcal{C}$} & 0.000 & 1.000 & 0.000 & 0.000 \\
{Only $\mathcal{K}$} & 0.000 & 0.000 & 1.000 & 0.000 \\
{Both $\mathcal{C}, \mathcal{K}$} & 0.000 & 0.000 & 0.000 & 1.000 \\
\midrule
\midrule
 \multicolumn{5}{c}{Inc, perturbations to $\mathcal{C} \text{ and } \mathcal{K}$}\\
 Robustness & {None} & {Only $\mathcal{C}$} & {Only $\mathcal{K}$} & {Both $\mathcal{C}, \mathcal{K}$} \\ \midrule
{None} & 1.000 & 0.000 & 0.000 & 0.000 \\
{Only $\mathcal{C}$} & 0.000 & 1.000 & 0.000 & 0.000 \\
{Only $\mathcal{K}$} & 0.000 & 0.000 & 1.000 & 0.000 \\
{Both $\mathcal{C}, \mathcal{K}$} & 0.000 & 0.000 & 0.000 & 1.000 \\
\bottomrule
\end{tabular}
    \label{tab:ss-inc}
\end{table}

\begin{table}[]
    \centering
    \caption{Result of pairwise comparison on the effect of accounting for different types of robustness for data set Hou under different perturbations (both random forest and neural network, both uniform and normal sampling distributions).
    The displayed $p$-values are obtained with the Mann-Whitney U test under Holm-Bonferroni correction and post Kruskal-Wallis test rejecting the null hypothesis with $p\text{-value} \ll 0.01$.}
\begin{tabular}{lcccc}
\toprule
 \multicolumn{5}{c}{Hou, perturbations to $\mathcal{C}$}\\
 Robustness & {None} & {Only $\mathcal{C}$} & {Only $\mathcal{K}$} & {Both $\mathcal{C}, \mathcal{K}$} \\ \midrule
{None} & 1.000 & 0.000 & 0.000 & 0.000 \\
{Only $\mathcal{C}$} & 0.000 & 1.000 & 0.261 & 0.006 \\
{Only $\mathcal{K}$} & 0.000 & 0.261 & 1.000 & 0.083 \\
{Both $\mathcal{C}, \mathcal{K}$} & 0.000 & 0.006 & 0.083 & 1.000 \\
\midrule
\midrule
 \multicolumn{5}{c}{Hou, perturbations to $\mathcal{K}$}\\
 Robustness & {None} & {Only $\mathcal{C}$} & {Only $\mathcal{K}$} & {Both $\mathcal{C}, \mathcal{K}$} \\ \midrule
{None} & 1.000 & 0.000 & 0.000 & 0.000 \\
{Only $\mathcal{C}$} & 0.000 & 1.000 & 0.000 & 0.000 \\
{Only $\mathcal{K}$} & 0.000 & 0.000 & 1.000 & 0.070 \\
{Both $\mathcal{C}, \mathcal{K}$} & 0.000 & 0.000 & 0.070 & 1.000 \\
\midrule
\midrule
 \multicolumn{5}{c}{Hou, perturbations to $\mathcal{C} \text{ and } \mathcal{K}$}\\
 Robustness & {None} & {Only $\mathcal{C}$} & {Only $\mathcal{K}$} & {Both $\mathcal{C}, \mathcal{K}$} \\ \midrule
{None} & 1.000 & 0.000 & 0.000 & 0.000 \\
{Only $\mathcal{C}$} & 0.000 & 1.000 & 0.000 & 0.000 \\
{Only $\mathcal{K}$} & 0.000 & 0.000 & 1.000 & 0.008 \\
{Both $\mathcal{C}, \mathcal{K}$} & 0.000 & 0.000 & 0.008 & 1.000 \\
\bottomrule
\end{tabular}
    \label{tab:ss-hou}
\end{table}

\begin{table}[]
    \centering
    \caption{Result of pairwise comparison on the effect of accounting for different types of robustness for data set Pro under different perturbations (both random forest and neural network, both uniform and normal sampling distributions).
    The displayed $p$-values are obtained with the Mann-Whitney U test under Holm-Bonferroni correction and post Kruskal-Wallis test rejecting the null hypothesis with $p\text{-value} \ll 0.01$.}
\begin{tabular}{lcccc}
\toprule
 \multicolumn{5}{c}{Pro, perturbations to $\mathcal{C}$}\\
 Robustness & {None} & {Only $\mathcal{C}$} & {Only $\mathcal{K}$} & {Both $\mathcal{C}, \mathcal{K}$} \\ \midrule
{None} & 1.000 & 0.000 & 0.000 & 0.000 \\
{Only $\mathcal{C}$} & 0.000 & 1.000 & 0.000 & 0.008 \\
{Only $\mathcal{K}$} & 0.000 & 0.000 & 1.000 & 0.000 \\
{Both $\mathcal{C}, \mathcal{K}$} & 0.000 & 0.008 & 0.000 & 1.000 \\
\midrule
\midrule
 \multicolumn{5}{c}{Pro, perturbations to $\mathcal{K}$}\\
 Robustness & {None} & {Only $\mathcal{C}$} & {Only $\mathcal{K}$} & {Both $\mathcal{C}, \mathcal{K}$} \\ \midrule
{None} & 1.000 & 0.000 & 0.000 & 0.000 \\
{Only $\mathcal{C}$} & 0.000 & 1.000 & 0.000 & 0.000 \\
{Only $\mathcal{K}$} & 0.000 & 0.000 & 1.000 & 0.760 \\
{Both $\mathcal{C}, \mathcal{K}$} & 0.000 & 0.000 & 0.760 & 1.000 \\
\midrule
\midrule
 \multicolumn{5}{c}{Pro, perturbations to $\mathcal{C} \text{ and } \mathcal{K}$}\\
 Robustness & {None} & {Only $\mathcal{C}$} & {Only $\mathcal{K}$} & {Both $\mathcal{C}, \mathcal{K}$} \\ \midrule
{None} & 1.000 & 0.000 & 0.000 & 0.000 \\
{Only $\mathcal{C}$} & 0.000 & 1.000 & 0.000 & 0.000 \\
{Only $\mathcal{K}$} & 0.000 & 0.000 & 1.000 & 0.014 \\
{Both $\mathcal{C}, \mathcal{K}$} & 0.000 & 0.000 & 0.014 & 1.000 \\
\bottomrule
\end{tabular}
    \label{tab:ss-pro}
\end{table}

\begin{table}[]
    \centering
    \caption{Result of pairwise comparison on the effect of accounting for different types of robustness for data set Rec under different perturbations (both random forest and neural network, both uniform and normal sampling distributions).
    The displayed $p$-values are obtained with the Mann-Whitney U test under Holm-Bonferroni correction and post Kruskal-Wallis test rejecting the null hypothesis with $p\text{-value} \ll 0.01$.}
\begin{tabular}{lcccc}
\toprule
 \multicolumn{5}{c}{Rec, perturbations to $\mathcal{C}$}\\
 Robustness & {None} & {Only $\mathcal{C}$} & {Only $\mathcal{K}$} & {Both $\mathcal{C}, \mathcal{K}$} \\ \midrule
{None} & 1.000 & 0.000 & 0.000 & 0.000 \\
{Only $\mathcal{C}$} & 0.000 & 1.000 & 0.000 & 0.000 \\
{Only $\mathcal{K}$} & 0.000 & 0.000 & 1.000 & 0.115 \\
{Both $\mathcal{C}, \mathcal{K}$} & 0.000 & 0.000 & 0.115 & 1.000 \\
\midrule
\midrule
 \multicolumn{5}{c}{Rec, perturbations to $\mathcal{K}$}\\
 Robustness & {None} & {Only $\mathcal{C}$} & {Only $\mathcal{K}$} & {Both $\mathcal{C}, \mathcal{K}$} \\ \midrule
{None} & 1.000 & 0.000 & 0.000 & 0.000 \\
{Only $\mathcal{C}$} & 0.000 & 1.000 & 0.000 & 0.000 \\
{Only $\mathcal{K}$} & 0.000 & 0.000 & 1.000 & 0.651 \\
{Both $\mathcal{C}, \mathcal{K}$} & 0.000 & 0.000 & 0.651 & 1.000 \\
\midrule
\midrule
 \multicolumn{5}{c}{Rec, perturbations to $\mathcal{C} \text{ and } \mathcal{K}$}\\
 Robustness & {None} & {Only $\mathcal{C}$} & {Only $\mathcal{K}$} & {Both $\mathcal{C}, \mathcal{K}$} \\ \midrule
{None} & 1.000 & 0.000 & 0.000 & 0.000 \\
{Only $\mathcal{C}$} & 0.000 & 1.000 & 0.000 & 0.000 \\
{Only $\mathcal{K}$} & 0.000 & 0.000 & 1.000 & 0.271 \\
{Both $\mathcal{C}, \mathcal{K}$} & 0.000 & 0.000 & 0.271 & 1.000 \\
\bottomrule
\end{tabular}
    \label{tab:ss-rec}
\end{table}

\end{document}